\documentclass{article}

\PassOptionsToPackage{numbers, compress}{natbib}

 \usepackage[preprint]{neurips_2024}

\usepackage[utf8]{inputenc} 
\usepackage[T1]{fontenc}    
\usepackage{hyperref}       
\usepackage{url}            
\usepackage{booktabs}       
\usepackage{nicefrac}       
\usepackage{microtype}      
\usepackage{xcolor}         

\usepackage{graphicx}
\usepackage{enumitem}
\usepackage{wrapfig}

\usepackage{subcaption}
\usepackage{tabularx}

\usepackage{algorithm}
\usepackage{algorithmic}

\definecolor{darkgreen}{HTML}{397d26}

\usepackage{amsmath}
\usepackage{amssymb}
\usepackage{mathtools}
\usepackage{amsthm}
\usepackage{amsfonts}       
\usepackage[capitalize,noabbrev]{cleveref}

\theoremstyle{plain}
\newtheorem{theorem}{Theorem}[section]

\newtheorem{lemma}[theorem]{Lemma}

\theoremstyle{definition}
\newtheorem{definition}[theorem]{Definition}
\newtheorem{assumption}[theorem]{Assumption}
\theoremstyle{remark}
\newtheorem{remark}[theorem]{Remark}

\usepackage[textwidth=3.5cm,textsize=tiny]{todonotes}
\newcommand{\highlight}[1]{\textcolor{black}{#1}}

\newcommand{\revision}[1]{\textcolor{black}{#1}}

\expandafter\def\expandafter\normalsize\expandafter{%
    \normalsize%
    \setlength\abovedisplayskip{5pt}%
    \setlength\belowdisplayskip{3pt}%
}

\usepackage{titlesec}
\titlespacing{\subsection}{0pc}{0.2em}{0.2em}
\titlespacing{\section}{0pc}{0pc}{0pc}

 \title{Federated Unlearning: a Perspective of Stability and Fairness
}

\author{Jiaqi Shao$^{1, 2}$, \quad Tao Lin$^{3}$, \quad Xuanyu Cao$^{1}$, \quad Bing Luo$^{2}$ \thanks{Corresponding author: Bing Luo, \texttt{\{bing.luo\}@dukekunshan.edu.cn}}\\
{\small$^1$The Hong Kong University of Science and Technology, 
Hong Kong SAR, China}\\
{\small$^2$Duke Kunshan University, Suzhou, China} \\
{\small$^3$Westlake University, 
Hangzhou, China}
}
\begin{document}

\maketitle
\newcommand{\ow}{\boldsymbol{\boldsymbol{w}}^{o}}
\newcommand{\uw}{\boldsymbol{{w}}^u}
\newcommand{\starw}{\boldsymbol{w}^*}
\newcommand{\w}{\boldsymbol{w}}
\newcommand{\boldw}[2]{\boldsymbol{w}^{#2}_{#1}}
\newcommand{\barboldw}[2]{{\boldsymbol{\bar{w}}}^{#2}_{#1}}
\newcommand{\E}[1]{\mathbb{E}\left[#1\right]}
\newcommand{\norm}[1]{\left\|#1 \right\|^2}

\begin{abstract}
\revision{This paper explores the side effects of federated unlearning (FU) with data heterogeneity, concentrating on global stability and local fairness. We introduce key metrics for FU assessment and investigate FU's inherent trade-offs. }
Furthermore, we formulate the FU problem through an optimization framework for managing side effects. Our key contribution lies in a comprehensive theoretical analysis of the trade-offs in FU and provides insights into data heterogeneity's impacts on FU. Leveraging these insights, we propose FU mechanisms to manage the trade-offs, guiding further development for FU mechanisms. We also empirically validate that our FU mechanisms effectively balance trade-offs, confirming insights derived from our theoretical analysis.
\end{abstract}

\setlength{\parskip}{4pt minus 2pt}

\section{Introduction}\label{sec:intro}
With the advancement of user data regulations, such as GDPR~\cite{regulation2018general} and CCPA~\cite{goldman2020introduction}, the concept of ``\textit{the right to be forgotten}" has gained prominence. 
It necessitates models' capability to forget or remove specific training data upon users' request, which is non-trivial since the models potentially memorize training data.
Intuitively, the most straightforward approach is to \textit{exactly retrain} the model from scratch without the data to be forgotten. However, this method is computationally expensive, especially for large-scale models prevalent in modern applications.
As a result, the machine unlearning (MU) paradigm is proposed to efficiently remove data influences from models~\cite{cao2015towards}. 
The effectiveness of unlearning, measured by \textit{verification} approaches, requires the unlearning mechanism to closely replicate the results of exact retraining without heavy computational burden.

Federated learning (FL) has gained attention in academia and industry with increased data privacy concerns by allowing distributed clients to collaboratively train a model while keeping the data local~\cite{kairouz2021advances}.
While MU offers strategies for traditional centralized machine learning context, federated unlearning (FU) introduces new challenges due to inherent data heterogeneity and privacy concerns in the federated context~\cite{wang2023federated}.  
Recent research of FU, such as \citet{gao_verifi_2022, che2023fast, DBLP:conf/iclr/0003SPRM23,liu2023survey}, mainly focused on verification and efficiency in FU, aligning with the main objectives of MU.
However, the inherent data heterogeneity in federated systems introduces new challenges: (i) different clients contribute differently to the global model, thus unlearning specific clients can lead to \textit{diverse impacts on model performance}; (ii) due to clients' diverse preferences for the global model, unlearning certain clients could result in \textit{unequal impacts on individuals}. 

\begin{figure}[ht]
    \centering
    \includegraphics[width=.6\linewidth]{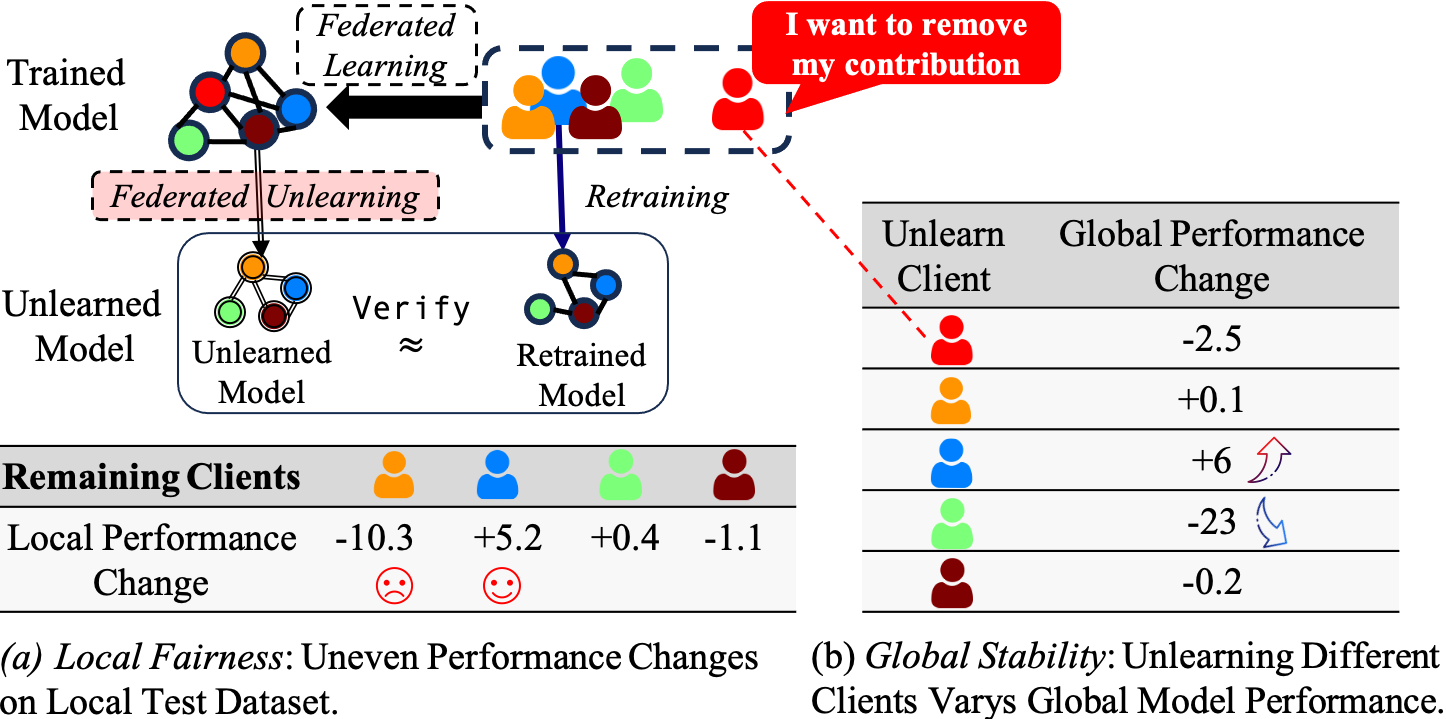}
        \caption{\textbf{Federated Unlearning and Its Side Effects.} (a) FL system of $5$ clients with non-IID training data. with one client requesting to be unlearned. (b) Unlearning each client through $5$ experiments. }
        \label{fig:motivated_example}
\end{figure}

\begin{figure}[htbp]
    \centering
        \includegraphics[width=.6\linewidth]{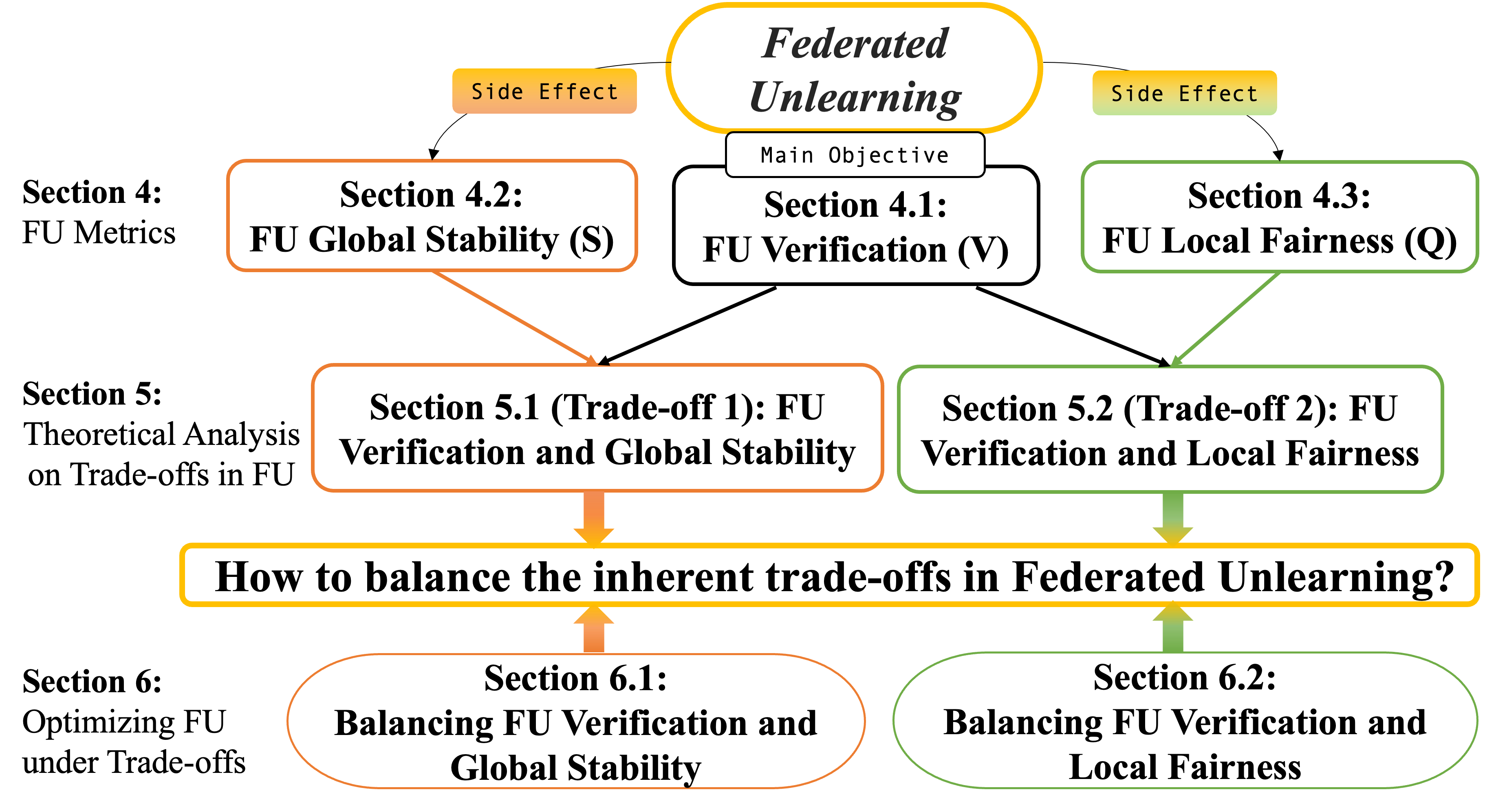}
        \caption{Outline of this paper.}
        \label{fig:paper_outline}
\end{figure}

\revision{\textbf{On the Side Effects of FU under heterogeneous data.}
\cref{fig:motivated_example} elaborates two key side effects posed by {data heterogeneity}.}
(1)~\underline{Global Stability}: Unlearning different clients leads to {different impacts} on the global model's performance, as depicted in \cref{fig:motivated_example}(b). 
This highlights the {\textit{``global stability"} concern in system performance};
(2)~\underline{Local Fairness}: 
As shown in \cref{fig:motivated_example}(a), FU can unequally impact remaining clients, where some benefit from unlearning, but others experience disadvantages.
It illustrates a {``\textit{local fairness}" concern, pertaining to the uniform distribution of utility changes\footnote{In this context, `utility' refers to an individual client's experienced performance of the global model.} among remaining clients after unlearning.} 

\revision{\textbf{FU undertakes the potential for problematic unbounded instability and unfairness}. Without controlling the instability \textit{or} unfairness, unlearning may potentially drive all remaining clients to leave the system, leading to catastrophic forgetting~\cite{liu2022continual} and problematic behavior among selfish clients that exploit the resources of others, akin to free-rider attacks~\cite{fraboni2021free}. 
 Thus, it is essential to consider the two \textit{trade-off} inherent in FU, including the \textit{FU verification vs. global stability} trade-off, as well as the \textit{FU verification vs. local fairness} trade-off.} 
These insights motivate us to ask:

\textit{\textbf{Q: }How can we assess the side effects of FU, and what are the theoretical and practical approaches to balancing the inherent trade-offs?
}

{To address the above question,
we construct a comprehensive theoretical framework for FU's side effects, which offers a rigorous understanding of how \textit{data heterogeneity} affects FU while balancing the trade-offs. We outline the structure of our paper in \cref{fig:paper_outline}\\
}\textbf{Contributions}. Our key contributions are summarized as follows:

\begin{enumerate}[nosep, leftmargin=12pt]
    \item \textbf{Qualitative Understanding of FU Metrics}: We introduce robust quantitative metrics for FU assessment, including \textit{verification metric}, \textit{global stability metric} and \textit{local fairness metric} (detailed in \cref{sec:system_model}). These metrics provide a foundation for comprehensive FU trade-off evaluations.
    \item \textbf{Theoretical Analysis of FU Trade-offs}: We present a theoretical analysis of the \textit{trade-offs} in FU (\cref{sec:trade-offs}). Under data heterogeneity, our results demonstrate challenges in balancing between FU verification and global stability, as well as between FU verification and local fairness.\looseness=-1
    \item \textbf{FU Mechanism and Theoretical Framework}: We propose a novel FU mechanism based on the theoretical framework, {encompassing optimization strategies and penalty methods} into balancing the tradeoffs within a verifiable FU context, as detailed in~\cref{sec:framework}. The experimental evaluation is provided in~\cref{appx:exp}.
\end{enumerate}

\section{\revision{Related Work}}
\textbf{Mechine Unlearning \& Federated Unlearning}.
Machine unlearning (MU) aims to remove specific data from a machine learning model, addressing the challenges in both effectiveness and efficiency of the unlearning process \cite{guo2020certified, DBLP:conf/icml/WuDD20, bourtoule2021machine, 9519428, Tarun_2023, DBLP:conf/icml/TarunCMK23, jia2023model, NEURIPS2023_062d711f}.
In FL, federated unlearning (FU) is proposed to address clients' right to be forgotten, including methods like rapid retraining~\cite{liu2022right}, subtracting historical updates from the trained model~\cite{wu2022federated}, subtracting calibrated gradients of the unlearn clients to remove their influence~\cite{liu2020federated,liu_federaser_2021}, and adding calibrated noises to the trained model by differential privcacy~\cite{zhang2023fedrecovery}.  

Existing FU methods
has focused on ensuring verifiable and efficient unlearning~\cite{liu_federaser_2021, liu2022right, fraboni2022sequential, gao_verifi_2022, jin2023forgettable, che2023fast}. 
The \textit{verification} of unlearning typically involves comparing the unlearned model, obtained through an unlearning mechanism, with a reference model using performance metrics such as accuracy and model similarity metrics~\cite{gao_verifi_2022}. 
Additionally, attack-based verification methods, such as membership inference attacks (MIA) and backdoor attacks (BA), are often employed in MU~\cite{nguyen_survey_2022} 
but are \highlight{not applicable} in federated systems due to privacy concerns.
However, none of them involves any rigorous consideration for \textit{data heterogeneity}, the main challenge in FL. 
As discussed in~\cref{sec:intro}, data heterogeneity can impact the global stability and local fairness of the FL system.
In this work, \highlight{we theoretically analyze how data heterogeneity impacts FU in~\cref{sec:trade-offs}.}

\textbf{Stability in FL and FU}.
\highlight{In FL, the concept \textit{performance stability} is introduced for describing robust model performance. 
This aspect of stability often involves defending against external threats, like alterations in the training dataset~\cite{yin2018byzantine, fang2020local, li2021ditto}. 
Unlike FL, FU is concerned with managing internal changes within the system for users' rights to be forgotten. 
Due to the inherent data heterogeneity of FL system, unlearning can lead to data distribution shifts and model performance change. 
Besides of reducing the performance change, it is also crucial to ensure unlearning is verifiable in FU.
In this work, we theoretically examine how data heterogeneity impacts stability in FU, and how to balance the trade-offs between FU verification and stability.}

\textbf{Fairness in FL and FU}. 
In FL, 
there are several works that have proposed different notions of fairness. 
The proportional fairness ensures whoever contributes more to the model can gain greater benefits~\cite{wang2020principled, yu2020fairness}. 
Additionally, the model fairness focuses on protecting
specific characteristics, like race and gender~\cite{gu2022modelfair}.
Furthermore, the \textit{performance fairness}~\cite{li2019fair, mohri_agnostic_2019, hao2021towards, li2021ditto, shi2023towards, pan2023fedmdfg} aims to reduce the \textit{variance of local test performance or utility} across all clients. 
The FL fairness is mainly influenced by data heterogeneity across participating clients.
\highlight{In FU, unlearning certain clients can lead to \textit{unequal} impacts on remaining clients due to data heterogeneity, as discussed in~\cref{sec:intro}.
In addition to heterogeneity among remaining clients,
we also account for the data heterogeneity between the removed clients and the remaining clients, 
making fairness analysis more challenging in FU compared to FL.
In this work, 
we establish performance fairness by the variance of \textit{utility changes} among remaining clients and further analyze how data heterogeneity impacts fairness in FU.}

\section{Preliminaries}

\begin{figure}[ht]
    \centering
    \begin{minipage}{0.48\textwidth}
    \centering
    \includegraphics[width=\linewidth]{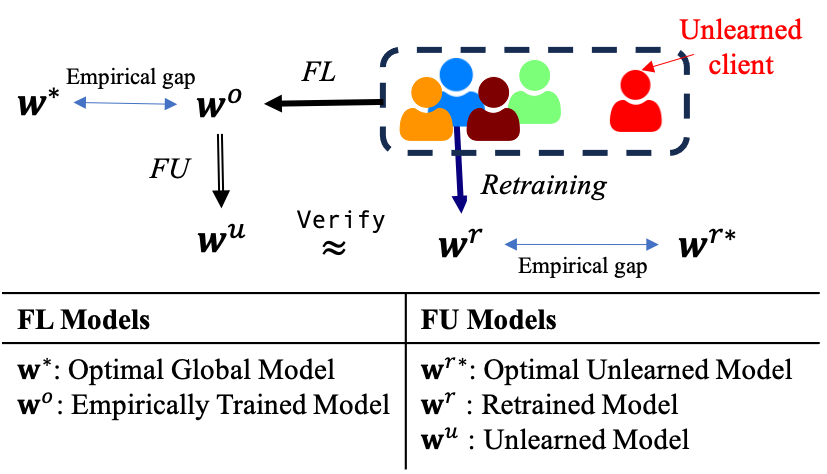}

    \caption{Key Model Notations ($\uw$ is the unlearned model from the FU mechanism; $\w^r$ is exact retrained model, the special case of $\uw$).}
    \label{fig:model_notations}
    \end{minipage}
    \begin{minipage}{0.5\textwidth}
    \begin{algorithm}[H]
    \caption{\revision{Federated Unlearning Mechanism $\mathcal{M}$}}
    \label{alg:fu}
    \begin{algorithmic}[1]
    \small
    \STATE {\bfseries Input:} $\ow$ (original FL trained model), $\mathcal{J}$ (clients to unlearn), $T$ (unlearning rounds)
    \STATE Initialize $\w_0 \gets \ow$
    \FOR{$t = 0$ to $T-1$}
    \STATE $\mathcal{S} \gets $ subset of remaining clients $\mathcal{N} \setminus \mathcal{J}$
    \FOR{client $i \in \mathcal{S}$ \textbf{in parallel}}
    \STATE $\w_i^{t+1} \gets \text{ClientUpdate}(\w_t, \mathcal{D}i)$ \ 
    \ENDFOR
    \STATE $\w_{t+1} \gets \sum_{i \in \mathcal{S}} p'_i\w_i^{t+1}$ \ 
    \ENDFOR
    \STATE \textbf{return} $\uw \gets \w_T$ \ 
    \end{algorithmic}
    \end{algorithm}
    \end{minipage}
\end{figure}

\textbf{Federated Learning (FL).}
Suppose there is a client set $\mathcal{N}$ ($|\mathcal{N}| = N$), contributing to FL training. Each client $i \in \mathcal{N}$ has a local training dataset with size $n_i$, and \highlight{the data is non-IID across different clients}.
The optimal FL model is defined as $\boldsymbol{w}^{*} = \arg \min_{\boldsymbol{w}} [F(\w) := \sum_{i \in \mathcal{N}} p_i f_i(\w)]$, where $f_i$ denotes the local objective function of client $i$ with the aggregation weight $p_i = \frac{n_i}{\sum_{k \in \mathcal{N}} n_k}$. 
The \textit{empirically trained} model $\ow$ is obtained by training all clients $i \in \mathcal{N}$, serving as the empirical approximation of the theoretical optimum $\boldsymbol{w}^{*}$ (summarized in~\cref{fig:model_notations}). 

\textbf{Federated Unlearning (FU).}
FU mechanisms aim to obtain an unlearned model \( \uw \) that removes the influence of client set $\mathcal{J \in \mathcal{N}}$ who requests to be forgotten from the original trained model~$\ow$. 
The optimal unlearned model could be defined as $\boldsymbol{w}^{r*} = \arg \min_{\boldsymbol{w}} [F_{-\mathcal{J}}(\boldsymbol{w}) := \sum_{ i \notin \mathcal{J}} p_i' f_i(\w)]$, 
where $p_i'$ is the normalized aggregation weight exclude unlearned clients, \textit{i}.\textit{e}., $p_i' = \frac{p_i}{1-P_\mathcal{J}}$ with $P_\mathcal{J} = \sum_{j \in \mathcal{J}} p_j$. 
The \textit{exact retrained} model $\w^r$ is obtained after retraining all remaining clients\footnote{\highlight{In this paper, the retraining employs FedAvg~\cite{pmlr-v54-mcmahan17a}.}}, 
serving as the empirical approximation of the theoretical optimum $\w^{r*}$ (summarized in~\cref{fig:model_notations}).

\revision{\textit{\textbf{Setting.}}} This paper focuses on an FU mechanism that starts with 
the original trained model $\ow$ and iteratively updates on the remaining clients' data ($i \notin \mathcal{J}$) as shown in~\cref{alg:fu}~\cite{wuerkaixi2024accurate}. 
The goal is to obtain an unlearned model $\uw$ that closely replicates the results of exact retraining on the remaining data distribution. 
This approach aligns with the perspective of continual learning for unlearning~\cite{shibata2021learning, liu2022continual, nguyen2022survey}.

\section{FU Metrics}\label{sec:system_model}
This section introduces quantitative metrics for evaluating FU mechanisms. 
In \cref{sec:fu}, we elaborate on the verification metric to \textit{verify the effectiveness} of the FU process. 
\cref{sec:stability_metric} assesses FU's impact on the system's global stability, quantifying how FU \textit{alters the global performance} of the model. 
Additionally, \cref{sec:fair_metric} evaluates FU's impact on local fairness, capturing how FU \textit{unequally impacts individuals} in remaining clients. 
{These metrics lay the foundation for our comprehensive framework that captures the inherent trade-offs in FU, such as balancing between FU verification and global stability, as well as FU verification and local fairness.}

\subsection{FU Verification}\label{sec:fu}
FU verification is to evaluate how well the unlearning mechanism effectively removes the data~\cite{yang2023survey}. 
Previous studies often employ the exact retrained model $\boldsymbol{w}^r$ as the evaluation benchmark~\cite{halimi2022federated, che2023fast}, 
but this benchmark suffers from variability due to randomness in the retraining process.
Thus, we employ the optimal unlearned model $\boldsymbol{w}^{r*}$ as a theoretical benchmark, allowing for reliable FU evaluations in theoretical analysis.
We introduce a verification metric as the performance disparity between the unlearned model and $\boldsymbol{w}^{r*}$:

\begin{definition}[FU Verification Metric, $V$]
Consider an FU mechanism $\mathcal{M}$ designed to remove specific clients' influence, resulting in the \textit{unlearned model} $\uw$.
The unlearning verification metric $V(\uw)$ quantifies the effectiveness of $\mathcal{M}$ and is defined as:
\begin{equation}
    V(\uw) = \mathbb{E} \left[  F_{-\mathcal{J}}(\uw)\right] - F_{-\mathcal{J}}(\boldsymbol{w}^{r*}),
\end{equation}
where $F_{-\mathcal{J}} (\cdot)$ measures performance over the \textit{remaining clients} after unlearning clients set $\mathcal{J}$. 
\revision{The expectation $\mathbb{E}$ is taken over the randomness
in obtaining $\uw$ via $\mathcal{M}$.}
\end{definition}

\subsection{Global Stability}\label{sec:stability_metric}
Effective unlearning often requires modification to the trained model, which potentially alters global performance and compromises the utility of other remaining clients, especially under data heterogeneity.
To measure the extent of performance stability in FU, 
we introduce the global stability metric as follows: 
\begin{definition}[Global Stability Metric, $S$]
Given an FU mechanism $\mathcal{M}$ and its resulting unlearned model $\uw$.
    The metric $S(\uw)$ evaluates the 
     global stability of $\mathcal{M}$, measuring the performance gap between the unlearned model $\uw$ and the optimal original FL model $\starw$: 
\begin{equation}
    S(\uw) = \mathbb{E} \left[ F(\uw)\right]  - F(\starw) \,.
\end{equation}
\end{definition} 
\highlight{This metric $S$ evaluates the stability of the FU process, facilitating understanding of theoretical analysis in~\cref{sec:trade_off_global_stability}.
}
\subsection{Local Fairness}\label{sec:fair_metric}

In FL, fairness can be associated with the consistency of model performance across different clients. 
Specifically, a model is considered fairer if its performance has a smaller variance across clients~\cite{DBLP:conf/iclr/LiSBS20}.

For FU, effectively unlearning certain clients can unequally impact remaining clients because {they have diverse preferences for the global model (data heterogeneity)}.
As demonstrated in~\cref{sec:intro}, some clients experience significant utility degradation after unlearning, potentially prompting their departure and further degrading system performance.
To measure this FU impact, we 
propose the local fairness metric as follows:
\begin{definition}[Local Fairness Metric, $Q$]
  Given an FU mechanism $\mathcal{M}$ and its resulting unlearned model $\uw$.
  The metric $Q(\uw)$ evaluates local fairness of $\mathcal{M}$, {assessing the unequal impact of FU on remaining clients}:
  \begin{equation}
\textstyle
     Q(\uw) = \sum_{i \notin \mathcal{J}} p_i' \left| \Delta f_i(\uw) - \overline{\Delta f} \right|, 
\end{equation}
where $\Delta f_i(\uw) = \mathbb{E} \left[ f_i(\uw)\right]  - f_i(\starw)$ represents the utility change for remaining client $i \notin \mathcal{J}$. 
\end{definition} 
The metric $Q$ captures FU's impact on utility changes experienced by remaining clients and further facilitates theoretical analysis of fairness implication in~\cref{sec:trade_off_local_fairness}. 

\section{Theoretical Analysis on Trade-offs in FU}\label{sec:trade-offs}
This section provides a theoretical analysis of the trade-offs in FU, particularly focusing on the balance between {FU verification and stability}, as well as {FU verification and fairness}, as outlined in \cref{sec:trade_off_global_stability} and~\ref{sec:trade_off_local_fairness}, respectively.
Our analyses underscore the inevitable costs and trade-offs that any practical FU mechanism must address, particularly in the presence of data heterogeneity. By identifying these trade-offs, we provide a principled foundation for developing heterogeneity-aware FU algorithms that manage the trade-offs.
To begin, we formally state the assumptions required for the theoretical analysis.

\begin{assumption}[Data Heterogeneity in FL]
\label{ass:data-heterogeneity}
Given a subset of remaining clients $\mathcal{S}$, the data heterogeneity among remaining clients can be quantified as follows: 
\begin{equation}
    \mathbb{E}_{i \in S}\left\| \nabla f_{i}(\boldsymbol{w}) - \nabla F_{-\mathcal{J}}(\boldsymbol{w})\right\|^2 \leq \zeta_{\mathcal{S}}^2 + \beta_{\mathcal{S}}^2 \|\nabla F_{-\mathcal{J}}(\boldsymbol{w})\|^2\,,
\end{equation}
where $\zeta_{\mathcal{S}}^2$ and $\beta_{\mathcal{S}}^2$ are parameters quantifying the heterogeneity. Here, $f_{i}$ represents the objective function of client $i$ in subset $\mathcal{S}$, and $F_{-\mathcal{J}}$ is the global objective function of the remaining clients.
\end{assumption}

{\cref{ass:data-heterogeneity} assumes the data heterogeneity with parameters $\zeta_{\mathcal{S}}, \beta_{\mathcal{S}}$, representing the degrees of data heterogeneity within the selected subset $\mathcal{S}$ of remaining clients, aligning closely with the framework presented by \citet{wang2021field}.}

\begin{assumption}[$\mu$-strong Convexity]
\label{ass:mu-convexity}
{Assume that local objective functions $f_i: \mathbb{R}^d \rightarrow \mathbb{R}$ are all $\mu$-strong convex. For any vectors $\mathbf{u}, \mathbf{v} \in \mathbb{R}^d$, $f_i$ satisfies the following inequality:
$f_i(\mathbf{u}) \geq f_i(\mathbf{v}) + \langle \nabla f_i(\mathbf{v}), \mathbf{u} - \mathbf{v} \rangle + \frac{\mu}{2} \|\mathbf{u} - \mathbf{v}\|^2$,
where $\mu > 0$ is the convexity constant.}
\end{assumption}

\begin{assumption}[$L$-smoothness]
\label{ass:l-smoothness}
{
Assume that local objective functions $f_i: \mathbb{R}^d \rightarrow \mathbb{R}$ are all $L$-smooth. For any vectors $\mathbf{u}, \mathbf{v} \in \mathbb{R}^d$, $f_i$ satisfies the following inequality:
$f_i(\mathbf{u}) \leq f_i(\mathbf{v}) + \langle \nabla f_i(\mathbf{v}), \mathbf{u} - \mathbf{v} \rangle + \frac{L}{2} \|\mathbf{u} - \mathbf{v}\|^2$,
where $L > 0$ is the Lipschitz constant of the gradient of $f_i$.}
\end{assumption}

\begin{assumption}[Bounded Variance]\label{ass:unbiased_gradient}
{Let $\xi_t^k$ be sampled from the $k$-th client's local data uniformly at random. The variance of stochastic gradients in each client is bounded at round $t$: $\mathbb{E}\left\|\boldsymbol{g}_k (\w_t)-\nabla f_k\left(\w_t\right)\right\|^2 \leq \sigma_{k, t}^2$ for $k=1, \cdots, N$, where $\boldsymbol{g}_k (\w_t) = \nabla f_k\left(\mathbf{w}_t, \xi_t\right)$.}
\end{assumption}

\begin{assumption}[Unlearning Clients' Influence]
\label{assm:client-proportion}
Let $P_\mathcal{J}$ denote the total aggregation weights of the clients in set $\mathcal{J}$ within FL, defined as $P_\mathcal{J} = \sum_{j \in \mathcal{J}} p_j$. For a client set $\mathcal{J}$ required for unlearning, we assume that 
$P_\mathcal{J} \leq \frac{1}{2}$.
\end{assumption}

\cref{ass:mu-convexity}-\ref{ass:unbiased_gradient} are commonly used for the FL convergence analysis, e.g., \citet{DBLP:conf/iclr/LiHYWZ20,wang2021field} and \cref{assm:client-proportion} assumes unlearned clients' aggregate weights do not exceed those of the remaining clients. 
\highlight{This is crucial to prevent catastrophic consequences, which could undermine the objectives of fairness and stability in FU.}

\subsection{Trade-off between FU Verification and Stability }\label{sec:trade_off_global_stability}
This section explores the trade-off between FU verification and global stability via the lower bound derived for verification~(\cref{lem:verification-bound}) and stability~(\cref{lem:stability-bound}).
Then, we formalize the trade-off characterized via the lower bounds in~\cref{thm:gradient-differences}. 
We provide all proofs for lemmas and theorems in the Appendix. 

\begin{lemma}[]
\label{lem:verification-bound}
Under Assumptions~\ref{ass:data-heterogeneity}-\ref{ass:unbiased_gradient} and given the number of unlearning rounds $T$ and the learning rate $\eta = \frac{1}{T \sqrt{\mu}}\sqrt{ \frac{\beta_{\mathcal{S}} - 1}{\min \left\{ \mu (\beta_{\mathcal{S}} - 1), L (\beta_{\mathcal{S}} -1) \right\} }}$,
the verification metric $V(\uw) = F_{-\mathcal{J}}(\uw) - F_{-\mathcal{J}}(\boldsymbol{w}^{r*})$ is lower bounded by:
\begin{equation}\label{eq:C_1}
  C_1 = \left(1 + {\frac{\beta_{\mathcal{S}}^2 - 1}{T}}\right) \left(\Delta F_{-\mathcal{J}} \left(\starw, \boldsymbol{w}^{r*} \right)\right.  \left. \, + \Delta F_{-\mathcal{J}} \left(\ow, \starw \right) \right) + \frac{1}{2LT} \left( \bar{\sigma^2} + \bar{\zeta^2} \right) \nonumber \,,
\end{equation}
where $\Delta F_{-\mathcal{J}} \left(\circ, \bullet \right) \!=\! F_{-\mathcal{J}}(\circ) \!-\! F_{-\mathcal{J}}(\bullet)$, $\bar{\sigma^2} = \frac{1}{T} \sum_{t=1}^T \sigma_t^2$, and $\bar{\zeta^2} =\frac{1}{T} \sum_{t=1}^T  \zeta_t^2$. Parameters $\beta_{\mathcal{S}}$ and $\bar{\zeta^2}$ characterize the data heterogeneity of remaining clients.
\end{lemma}

\begin{remark}[]
The effectiveness of FU, as measured by $V(\uw)$, is hindered by its lower bound $C_1$ in \cref{eq:C_1} with several factors:
\begin{itemize}[nosep, leftmargin=12pt]
\item \textit{Computational Complexity}: More unlearning rounds $T$ typically indicate convergence towards the optimal unlearned model $\w^{r*}$, characterized by a tighter lower bound $C_1$. However, the computational complexity grows as the number of unlearning rounds $T$ increases. 
\item \textit{Data Heterogeneity among Remaining Clients:} 
A high data heterogeneity ($\beta_{\mathcal{S}}^2, \bar{\zeta^2}^2$) among remaining clients can amplify $C_1$.
Therefore, under unlearning rounds $T$, the more heterogeneous among remaining clients, the more challenging it is to achieve effective unlearning by the increased lower bound $C_1$ of $V$. 
\item \textit{Data Heterogeneity Between Remaining and Unlearned Clients:} The discrepancy $\Delta F_{-\mathcal{J}} \left(\starw, \boldsymbol{w}^{r*} \right)$ implies data heterogeneity between remaining and unlearned clients. 
A high \textit{heterogeneity} enlarges $C_1$, thereby potentially compromising FU verification $V$. 
Conversely, suppose the data is \textit{homogeneous} between these two groups, $C_1$ can be diminished, as removing homogeneous data does not significantly alter the overall data distribution ($\Delta F_{-\mathcal{J}} \left(\starw, \boldsymbol{w}^{r*} \right) \approx 0$).
\end{itemize}
\end{remark}

\begin{lemma}[]
\label{lem:stability-bound}
Under Assumptions~\ref{ass:mu-convexity}, \ref{assm:client-proportion} and consider $T \geq \frac{\mu}{\eta^2}$ unlearning rounds. The global stability metric $S(\uw) = \mathbb{E} \left[ F(\uw)\right]  - F(\starw)$ is {bounded below} by $C_2$:
\begin{align}\label{eq:s_lb}
C_2 = \frac{P_\mathcal{J} \eta T}{2} \| \nabla F_{-\mathcal{J}}(\ow) - \nabla F_\mathcal{J} (\ow) \|^2  + \delta \,,
\end{align}
where $\delta = F(\ow) - F(\starw)$ represents the empirical risk minimization (ERM) gap in the original FL training. 
\end{lemma}

\begin{remark}[]\label{remark:stab}
Maintaining global stability $S(\uw)$ poses challenges due to the lower bound $C_2$ established in \cref{eq:s_lb}, which is influenced by the following factors: 
\begin{itemize}[nosep, leftmargin=12pt]
\item \textit{Unlearned Clients' Influence:}
The higher aggregation weight of unlearned clients $P_\mathcal{J}$ implies their substantial influence on the original model. Consequently, their removal has a greater impact on the model's performance, as reflected by increasing $C_2$.
\item \textit{Data Heterogeneity Between Remaining and Unlearned Clients:} 
 $\| \nabla F_{-\mathcal{J}}(\ow) - \nabla F_\mathcal{J} (\ow) \|^2$ measures the objectives divergence between remaining and unlearned clients. 
A larger value of this term indicates higher heterogeneity between the two groups, contributing to increased $C_2$ and thereby increasing instability. 
\item \textit{Unlearning Rounds}: {Increasing unlearning rounds $T$ can enhance unlearning effectiveness as discussed in \cref{lem:verification-bound}. 
However, the growth of $T$, particularly with divergent objectives $\| \nabla F_{-\mathcal{J}}(\ow) - \nabla F_\mathcal{J} (\ow) \|^2$, intensifies instability by increasing $C_2$.}
\end{itemize}
\end{remark}
 
\begin{theorem}\label{thm:gradient-differences}
Let Assumptions~\ref{ass:data-heterogeneity}-\ref{assm:client-proportion} hold, and given an original trained model $\ow$ with FU mechanism $\mathcal{M}$. 
Consider the learning rate $\eta = \frac{1}{T \sqrt{\mu}}\sqrt{ \frac{\beta_{\mathcal{S}} - 1}{\min \left\{ \mu (\beta_{\mathcal{S}} - 1), L (\beta_{\mathcal{S}} -1) \right\} }}$, 
the sum of the FU verification metric and the global stability metric $V(\uw) + S(\uw)$ is bounded below by:
\begin{align}
     V(\uw) + S(\uw) \geq C_s = \frac{P_\mathcal{J} }{\sqrt{2} \mu} \| \nabla F_{-\mathcal{J}}(\ow) - \nabla F_\mathcal{J} (\ow) \|^2 + \delta + C_1 \,,
\end{align}
where $C_1$ is defined in \cref{lem:stability-bound} and $C_s$ is constant.
\end{theorem}

\highlight{\cref{thm:gradient-differences} illustrates a fundamental \textit{trade-off} in FU: \textit{effectively unlearning clients ($V$) while maintaining the stability of the global model's performance ($S$)}. This trade-off is determined by the divergence between $\starw$ and $\w^{r*}$. Specifically, achieving effective FU and stability is not feasible under the substantial divergence between $\starw$ and $\w^{r*}$.}

\subsection{Trade-off between FU Verification and Fairness }\label{sec:trade_off_local_fairness}
{This section delves into the trade-off between FU verification and local fairness among the remaining clients}. 
We introduce the following \cref{thm:trade-off-lf-cu}, which quantifies this trade-off by a lower bound for the cumulative effect of verification and fairness. The lower bound is determined by the optimality gap, which is defined as the disparity between the performance of the optimal unlearned model $\boldsymbol{w}^{r*}$ and the local optimal models for each remaining client $\boldsymbol{w_i}^*$. 

\begin{theorem}[Trade-off between Local Fairness and Effective Unlearning]
\label{thm:trade-off-lf-cu}
Within FU, the sum of the unlearning verification metric and the local fairness metric is bounded below by a constant $C_q$:
\begin{equation}
\textstyle
   2 V(\uw) + Q(\uw) \geq C_q= F_{-\mathcal{J}}^* - \sum_{i \notin \mathcal{J}} p_i' f_i(\boldsymbol{w_i}^*) \,,
\end{equation}
where $\boldsymbol{w_i}^*$ denotes the local optimal model for client $i$. 
\end{theorem}

\begin{remark}[]\label{remark:fairness}
The lower bound $C_q$ underscores 
another fundamental {trade-off} in FU: \textit{the balance between effectively unlearning ($V$) and maintaining fairness among the remaining clients ($Q$)}.
If $C_q$ is large, optimizing either metric could compromise the other. 
The challenges of balancing this trade-off primarily arise from data heterogeneity:
\begin{itemize}[nosep, leftmargin=12pt]
\item \textit{Data Heterogeneity among Remaining Clients:} 
When data distribution is \textit{homogeneous} among remaining clients, each client's optimal model ($\w_i^{*}$ for $\forall i \notin J$) is identical with the optimal unlearned model ($\w_i^{*} = \w^{r*}$). Thus, data homogeneity reduces $C_q$ to $0$, indicating FU verification and fairness can be achieved simultaneously.
In contrast, higher \textit{heterogeneity} means divergent optimal models for different clients, thus increasing $C_q$ and posing greater challenges in balancing this trade-off. 
\item \textit{Data Heterogeneity Between Remaining and Unlearned Clients}: 
As discussed in \cref{lem:verification-bound}, a high heterogeneity between remaining and unlearned clients increases lower bound $C_1$ for the unlearning verification metric $V$. 
With a constant $C_q$, a larger $V$ typically leads to a reduced fairness metric $Q$. 
It indicates that under \textit{higher heterogeneity}, fairness is enhanced for the remaining clients after unlearning.
The enhanced fairness is because unlearning divergent clients $\mathcal{J}$ aligns the FU optimal model $\w^{r*}$ more closely to remaining clients than the original FL optimal model $\starw$.
Conversely, under \textit{homogeneity} between two groups, unlearning reduces $V$ (in \cref{lem:verification-bound} $\Delta F_{-\mathcal{J}} \left(\starw, \boldsymbol{w}^{r*} \right) = 0$), and thereby, the fairness metric $Q$ primarily depends on data heterogeneity among remaining clients.

\end{itemize}
\end{remark}

\subsection{Trade-off in Verification, Fairness and Stability}

\highlight{As demonstrated in \cref{remark:fairness}, given high heterogeneity between remaining and unlearned clients, unlearning can improve fairness for the remaining clients because the FU optimal model more closely aligns with the remaining clients. 
However, under this heterogeneity, the stability of the federated system is compromised in FU, as discussed in \cref{remark:stab}. 
This underscores a fundamental trade-off that enhancing fairness compromises the system's stability. 
For future work, we will delve into the complex interplay in FU for balancing FU verification, stability, and fairness.
}

\section{{Optimizing FU under Trade-offs}}\label{sec:framework}

In the previous section, we examine the inherent trade-offs involving FU verification and their challenges. 
To balance these trade-offs, 
this section introduces our FU mechanisms developed within an optimization framework\footnote{\highlight{These FU mechanisms are grounded in approximate unlearning, which gives tolerance on effectiveness.}}.
\begin{figure}[ht]
    \centering
    \vspace{-1em}
    \begin{minipage}{0.48\textwidth}
        \begin{algorithm}[H]
            \caption{\revision{FU Mechanism Design under \\Global Stability Constraint.}}
            \label{alg:fedunlearn}
            \begin{algorithmic}[1]
            \small
            \STATE {\bfseries Input:} $\ow$ (original model), $\mathcal{J}$ (unlearn clients), $T$ (unlearn rounds), $E$ (local epochs), $\eta_l, \eta_g$ (learning rates), $\lambda$ (stability penalty)
            \STATE Initialize $\w_0 \gets \ow$
            \FOR{$t = 0$ to $T-1$}
                \STATE $\mathcal{S} \gets $ subset of remaining clients $\mathcal{N} \setminus \mathcal{J}$
                \STATE \textbf{for} {client $i \in \mathcal{S}$ \textbf{in parallel}}
                    \STATE $\quad \boldw{i}{(t, E)} \gets \text{LocalTrain}(\w_t, \mathcal{D}_i, E, \eta_l)$ \\
                \item[] \phantomsection {// \textsf{1. Federated Aggregation}.}
                \STATE $\barboldw{}{(t, E)} \gets \sum_{i \in \mathcal{S}} \alpha_i \boldw{i}{(t, E)}$  \\
                \vspace{0.2em}\item[] \phantomsection {// \textsf{2. Gradient Correction.}}
                \STATE $\boldsymbol{h}^t \gets \lambda (1-P_\mathcal{J}) \boldsymbol{g}_{\mathcal{S}}^t + \lambda P_\mathcal{J} \boldsymbol{\hat{g}}_\mathcal{J}^t$ \\
                \STATE $\boldsymbol{g}_c^t \gets \boldsymbol{h}^t - \text{Proj}_{\boldsymbol{g}_{\mathcal{S}}^t} \boldsymbol{h}^t$ \\
                \STATE $\w_{t+1} = \barboldw{}{(t+1, 0)} \gets \barboldw{}{(t, E)} - \eta_g \boldsymbol{g}_c^t$ \\
            \ENDFOR
            \STATE \textbf{return} $\uw \gets \w_T$ // Unlearned model \\
            \end{algorithmic}
            \end{algorithm}
    \end{minipage}
    \hfill
    \begin{minipage}{0.5\textwidth}

        \begin{algorithm}[H]
            \caption{\revision{ FU Mechanism Design under \\Local Fairness Constraint.}}
            \label{alg:pffl}
            \begin{algorithmic}[1]
            \small
            \STATE {\bfseries Input:} $\ow$, $\mathcal{J}$, $T$, $E$, $\eta$, $\epsilon$ (fairness threshold), $\Lambda$ (fairness penalty)
            \STATE Initialize $\boldsymbol{\mu} = \boldsymbol{0}$, $\w_0 \gets \ow$
            \FOR{$t = 0$ to $T-1$}
                \STATE $\mathcal{S} \gets $ subset of remaining clients $\mathcal{N} \setminus \mathcal{J}$
                \FOR{client $i \in \mathcal{S}$ \textbf{in parallel}}
                    \STATE $r_i \gets f_i(\w_t) - f_i(\ow)$ \\
                    \STATE \textbf{for} each $\tau =1, \cdots, E:$   \\
                    \STATE $\quad \w_i^{(t, \tau+1)} \!\gets\!  \w_i^{(t, \tau)} \!-\! \eta \nabla f_i(\w_i^{(t, \tau)})  -{\mu_i \eta} ( \nabla f_i(\w_i^{(t, \tau)}) - f_i(\ow))  \,$
                \ENDFOR
    
                \STATE $\w_t = \barboldw{}{(t+1, 0)} \gets \sum_{i \in \mathcal{S}} \alpha_i\w_i^{(t, E)}.$
                \STATE If {$\max_i {r}_i \leq \epsilon$}: \textbf{return} $\uw \gets \w_t$ \\
                \STATE  Update $\mu_i=\Lambda \frac{\exp \left(r_i\right)}{1+\sum_{i \notin \mathcal{J}} \exp \left(r_i\right)}$.    
            \ENDFOR
            \STATE \textbf{return} $\uw \gets \w_T$ \\
            \end{algorithmic}
            \end{algorithm}
    \end{minipage}
\end{figure}

\subsection{\revision{Balancing FU Verification and Global Stability}}\label{sec:sol-stab}
In FU, maintaining global stability is crucial for ensuring the overall performance and reliability of the federated system throughout the unlearning process. 
However, as explored in~\cref{sec:trade_off_global_stability}, a trade-off exists between FU verification and global stability.
To manage this trade-off, we propose an FU mechanism utilizing a penalty-based approach and gradient correction techniques (\cref{alg:fedunlearn}).
We also theoretically demonstrate the convergence of our method, with further experimental validation in~\cref{appx:exp}.

\textbf{FU Mechanism Design}: 
{To balance stability during FU, 
we formulate the optimization problem for unlearning as} \textbf{P1: } $\min_{\w} V(\w)+ \lambda S(\w)$. 
By adjusting $\lambda$, we can manage the trade-off between these two objectives, allowing for a flexible approach to specific stability requirements of FU.

By the definition of stability metric $S(\w)$, we have: 
$S(\w) = \mathbb{E} \left[ F(\w)\right]  - F(\starw) = 
\mathbb{E} \left[ F(\w)\right] - F(\ow) + \delta,
$
where $\delta =  F(\ow) - F(\starw) $. 
Consequently, solving \textbf{P2}: $\min_{\w} F_{-\mathcal{J}} (\w) + \lambda \left( F (\w) - F(\ow) \right)$ optimizes \textbf{P1}. 
However, in \textbf{P2}, {optimizing the global objective $F(\w)$ \textit{among all clients} is untraceable in FU as it cannot involve unlearned client $j \in \mathcal{J}$ in unlearning process.}

To address \textbf{P2}, we consider the approximate problem \textbf{P3} by minimizing the  quadratic upper bound for $F_{\mathcal{J}} (\w)$ at $\w_0$: 
   $\textbf{P3: } \min_{\w} H(\w) := [F_{-\mathcal{J}} (\w) + \tilde{h} (\w)]$,
where 
$\tilde{h} (\w) = (1-P_\mathcal{J}) \lambda F_{-\mathcal{J}} (\w) + 
P_\mathcal{J} \lambda \left( \langle \nabla F_\mathcal{J} (\ow), \w - \ow \rangle + \frac{L}{2} \norm{\w - \ow} \right)$.

To address \textbf{P3}, we propose an FU mechanism operating two steps during each unlearning round $t$:
\begin{enumerate}[nosep, leftmargin=12pt]
    \item \textit{\underline{Federated Aggregation}}: The remaining client $i \in \mathcal{S}$ performs local training over $E$ epochs with learning rate $\eta_l$ to obtain $\boldw{i}{(t, E)}$. 
Then, the server aggregates $\{\boldw{i}{(t, E)}\}_{i \in \mathcal{S}}$ for the global model
    \(
        \barboldw{}{(t, E)} = \sum_{i \in \mathcal{S}} \alpha_i \cdot \boldw{i}{(t, E)} 
    \), where $\alpha_i$ is the weight for client $i$.
    \item \textit{\underline{Global Correction}}: Following the aggregation, the server applies a gradient correction to $\barboldw{}{(t, E)}$.  
    Specially, 
    the server compute $\boldsymbol{h}^t  = \lambda (1-P_\mathcal{J}) \boldsymbol{g}_{\mathcal{S}}^t + \lambda P_\mathcal{J} \boldsymbol{\hat{g}}_\mathcal{J}^t$,\footnote{For simplicity, denote $\boldsymbol{h}^t := \boldsymbol{h} (\barboldw{}{(t, E)})$}
    where $\boldsymbol{\hat{g}}_\mathcal{J}^t (\w) = \nabla F_{-\mathcal{J}} (\ow) + L (\w - \ow)$. 
    The correction term $\boldsymbol{g}_c^t$ is then obtained by projecting $\boldsymbol{h}^t$ onto the tangent space of the aggregated gradient $\boldsymbol{g}_{\mathcal{S}}$: \(
        \boldsymbol{g}_c^t = \boldsymbol{h}^t - \text{Proj}_{\boldsymbol{g}_{\mathcal{S}}} \boldsymbol{h}^t.
    \)
    The global model is updated for the next round:
    \(
    \barboldw{}{(t+1, 0)} = \barboldw{}{(t, E)} - \eta_g \boldsymbol{g}_c^t,
    \)
    where $\eta_g$ is the learning rate for the gradient correction.
\end{enumerate}

The FU mechanism thus iteratively updates the global model by $\barboldw{}{(t+1, 0)} = \boldw{}{(t, 0)} - \eta_l \boldsymbol{g}_{\mathcal{S}}^t - \eta_g \boldsymbol{g}_c^t$. 

\textbf{Theoretical Analysis}: 
Now, we delve into the convergence of the proposed FU mechanism, ensuring its reliability in FU. 
{We also conduct theoretical analysis to determine the upper bound for the verification metric $V$, which is essential for verifying the effectiveness of unlearning}.
To begin, we formally state the assumptions required for our main results.

\begin{assumption}{}\label{ass:bounded-gradient}
    The gradients of local objectives are bounded, \textit{i}.\textit{e}., $\| \nabla f_i(\boldsymbol{w}) \| \leq G$ for all $i$.
\end{assumption}  

\begin{assumption}{}\label{ass:removed_remaining_heterogeneity}
    The heterogeneity between the unlearned clients $\mathcal{J}$ and the remaining clients 
  ${\norm{ \E{\boldsymbol{\hat{g}}_\mathcal{J}^t (\w)} - \nabla F_{-\mathcal{J}} (\w)}} \leq \zeta^{\prime 2} + \beta^{\prime 2} \norm{\nabla F_{-\mathcal{J}} (\w)}$,
$\zeta^{'2}$ and $\beta^{'2}$ indicate heterogeneity between unlearned and remaining clients.
\end{assumption}
\vspace{-0.3em}
The following lemma derives the upper bound on the expected norm of gradient correction, and then 
we establish the convergence theorem of our proposed FU mechanism.
\begin{lemma}\label{lemma:bounded_gradient_correction}
    Under \cref{ass:removed_remaining_heterogeneity}, the expected norm of the gradient correction at unlearning round $t$ is bounded:$
\E{\norm{\boldsymbol{g_c} (\barboldw{}{(t, E)})}} \leq \phi \left( \zeta^{\prime 2}
+ (\beta^{\prime 2} + 1) \norm{\nabla F_{-\mathcal{J}} (\barboldw{}{t, E})}\right)$,
where $\phi = \lambda^2 P_\mathcal{J}^2 (1+\cos^2 \theta)$, $\cos^2 \theta$ represents the similarity in objectives between remaining and unlearned clients.
\end{lemma}

\begin{theorem}[Convergence]\label{thm:converge_stab_bala}

   Let Assumptions~\ref{ass:data-heterogeneity}-\ref{ass:unbiased_gradient}, \ref{ass:bounded-gradient} and \ref{ass:removed_remaining_heterogeneity} hold, we consider an FU mechanism with diminishing step size $\eta_l = \frac{\beta}{2(t+\gamma)}$ 
    for some $\beta > \frac{1}{\mu}$ and $\gamma > 0$, such that $\eta_l \leq \frac{1}{4L}$.
    The convergence result after $t$ rounds is:
\begin{small}
\vspace{-.5em}
\begin{align}\label{eq:converge_bound}
\E{H \!\left(\bar{\mathbf{w}}^{(t+1, 0)} \right)}\! \!-\! H^* \leq 
L \frac{v}{\gamma+t} + \frac{1}{2L} P_{\mathcal{J}}^2 \lambda^2 G^2 + \frac{L}{2} \left( \norm{\w^{r*} - \ow}  + \norm{\boldw{}{r*} - \w^{H*}}\right)\,,
\end{align}
\end{small}%
Where $\w^{H*} = \text{argmin}_{\w} H(\w)$, and $ v\!=\!\max \left\{\frac{\beta^2 B}{\beta \mu-4},(\gamma+1) \norm{\boldw{}{r*} - \ow} \right\}$.
\end{theorem}

Our approach introduces additional complexity in the convergence analysis compared to that of \citet[Theorem 1]{DBLP:conf/iclr/LiHYWZ20} due to incorporating a gradient correction in FU, as detailed in \cref{apx:proof_converge_stab_bala}
This complexity is reflected in $B$ with additional components: $2 \phi \left(\frac{\eta_g}{\eta_l}\right)^2 ((\beta^{\prime 2} + 1)  G^2 +  \zeta^{\prime 2})$. 
It highlights {two insights}:
\begin{enumerate}[nosep, leftmargin=12pt]
    \item A high \textit{data heterogeneity} between remaining and unlearned clients, indicated by $\beta^{\prime 2}$ and $\zeta^{\prime 2}$, \textit{increases unlearning rounds} $T$ for FU convergence;
    \item The term $\phi = \lambda^2 P_\mathcal{J}^2 (1+\cos^2 \theta)$ in the convergence bound indicates that larger influence of unlearned clients (characterized by $P_\mathcal{J}$) and larger stability penalties ($\lambda$) increase rounds $T$ needed for FU convergence.
\end{enumerate}

In the special case of \textit{homogeneity} between remaining and unlearned clients, where the original model and the optimal unlearned model are ideally aligned ( $\norm{\boldw{}{r*} - \starw} = 0$, $\beta^{\prime 2} = \zeta^{\prime 2} = 0$, $\cos \theta = 1$, and $P_\mathcal{J} = 0.5$), the additional term in $B$ reduces to $\lambda^2 \left({\eta_l}/{\eta_g}\right)^2   G^2$, and $\norm{\boldw{}{r*} - \w^{H*}} = 0$ in \cref{eq:converge_bound} since $\boldw{}{r*}$ is the feasible solution of \textbf{P3}. 
{In this scenario, handling stability in FU is straightforward for the tight bound.} 
Conversely, our mechanism reduces the convergence bound in \textit{heterogeneous} settings, characterized by orthogonal gradient correction to remaining clients’ gradients ($\cos \theta = 0$). This indicates we effectively adapt to this heterogeneity. 

{Next, we verify unlearning in our FU mechanism by a theoretical upper bound on $V$}. The additional assumptions and lemmas are stated as follows:
\begin{assumption}\label{ass:gradient_connection}
For each round $t$, the norm of the gradient after $E$ epochs is bounded by the gradient at the start of the round $t$,
$ {\norm{\nabla F_{-\mathcal{J}} (\barboldw{}{(t, E)})} \leq \epsilon \norm{\nabla F_{-\mathcal{J}} (\barboldw{}{(t, 0)})}}$.
\end{assumption}

\begin{lemma}
    Under \cref{lemma:bounded_gradient_correction} and \cref{ass:gradient_connection},
    the expected norm of the gradient correction at unlearning round $t$ is bounded:
    \(\E{\norm{\boldsymbol{g_c} (\barboldw{}{(t, E)}) - \nabla F_{-\mathcal{J}} (\barboldw{}{(t, 0)})}} 
    \leq \zeta^{'' 2} + \beta^{'' 2} \norm{\nabla F_{-\mathcal{J}} (\barboldw{}{(t, 0)})}  \)
    , where \( 
    \zeta^{''2} =  \phi \zeta^{\prime 2},   \beta^{'' 2} = \phi \epsilon \beta^{\prime 2} + \phi \epsilon + 1  
    \).
\end{lemma}

\begin{theorem}[{Verifiable Unlearning}]\label{thm:verifiable-unlearning}
Under Assumptions~\ref{ass:data-heterogeneity}-\ref{ass:unbiased_gradient}, and Assumptinos~\ref{ass:bounded-gradient}-\ref{ass:gradient_connection},
taking $\eta_l = \eta_g = \frac{2}{LT}$, and $T \geq \max\left\{2\beta_\mathcal{S}^2 + 2, \frac{1 + \Delta}{4L}, \frac{1}{2}(\beta''^2 + 1),\frac{1 + \Delta'}{L}\right\}$,
where $\Delta = \sqrt{\max\left\{0, 1 - 16L(\beta_\mathcal{S}^2 + 1)\right\}}$ and $\Delta' = \sqrt{\max\left\{0, 1 - L(\beta''^2 + 1)\right\}}
$. 

Then, the FU verification guarantees
$V = \mathbb E \left[F_{-\mathcal{J}} (\uw) -  F_{-\mathcal{J}} (\starw) \right] \leq \chi_1 + \chi_2 $, where $
\chi_1 = \frac{1}{2}\left(1 - \frac{1}{2LT} + \frac{\beta_\mathcal{S}^2 + 1 }{LT^2}\right)^T D + 
\frac{(\sigma^2 + \zeta_\mathcal{S}^2) }{2 LT}, \quad \chi_2  = \frac{1}{2}\left(1- \frac{1}{LT} + \frac{\beta^{''2} + 1 }{LT^2} \right)^TD + \frac{\zeta''^2}{2LT}. \,$
Here, $D =F_{-\mathcal{J}}(\ow) - F_{-\mathcal{J}} (\boldw{}{r*}), {\sigma^2 = \sum_{i \in \mathcal{S}} \alpha_i^2 \sigma_{i, t}^2}$
\end{theorem}

From \cref{thm:verifiable-unlearning}, the FU verification metric 
$V$ is {primarily determined by two factors}: 
\begin{enumerate}[nosep, leftmargin=12pt]
    \item \textit{Data Heterogeneity among Remaining Clients} ($\beta_\mathcal{S}^2$): Within $\chi_1$, a higher data heterogeneity necessitates more unlearning rounds $T$ to lower $V$ for effective unlearning; 
    \item \textit{Impact of Global Gradient Correction}: Within $\chi_2$, $\beta''^2$ encapsulates stability penalty ($\lambda$), unlearned clients' influence (captured by $P_\mathcal{J}$), and the data heterogeneity between remaining and unlearned clients. Increasing either of them requires more rounds $T$ to lower $V$.
\end{enumerate}
\highlight{Additionally, \cref{thm:verifiable-unlearning} highlights future adaptive strategies with client sampling or reweighting to reduce heterogeneity and variance of sampled remaining clients $\mathcal{S}$ in FU.}

\subsection{FU for Balancing Local Fairness}\label{sec:sol-fair}
As discussed in Section~\ref{sec:trade_off_local_fairness}, 
FU can lead to uneven impacts across different clients
due to data heterogeneity. 
To address this, we propose an optimization framework to minimize the unlearning objective with fairness constraints, ensuring that unlearning does not unequally harm any remaining clients. We highlight our contribution to a theoretical and practical groundwork for balancing fairness and verification in FU, and providing insights for future adaptive strategies.
\begin{align*}
    \textbf{P4: } & \min_{\w} F_{-\mathcal{J}} (\w).  
 \quad \text{ s.t. } \Delta f_i = f_i(\w) - f_i(\ow) \leq \epsilon, \forall i \notin \mathcal{J} \,,
\end{align*}
\textbf{FU Mechanism Design:} To solve the above problem, we introduce \cref{alg:pffl}, adapting the saddle point optimizations as in \cite{agarwal2018reductions,hu2022fair},
using a Lagrangian multiplier $\lambda_i$ for each constraint:
    \textbf{P5: } $\min_{\w} \max_{\boldsymbol{\lambda} \in \mathcal{R}^{Z}, \| \boldsymbol{\lambda}\| \leq \Lambda} F_{-\mathcal{J}} (\w) + \boldsymbol{\lambda}^{\top} \boldsymbol{r}(\w)$,
where $\boldsymbol{r}(\w) = \left[ \Delta f_i - \epsilon \right]_{i \notin \mathcal{J}}$. 
The detailed experimental examination are provided in \cref{appx:exp}.
\begin{lemma}\label{thm:fairness-trade-off}
Assume $\|\mathbf{r}\|_{\infty} \leq \rho$, and suppose $\nu = 2 \rho^2 \Lambda$, achieving a $\nu$-approximate saddle point of \textbf{P5} 
     requires $T \geq \frac{1}{\nu(\gamma+1)-2 \kappa \mathcal{C}}\left( \frac{M}{\nu} + 2 \kappa \mathcal{C}(\gamma-1)\right)$, 
    where $M$ is a constant and
     $\mathcal{C}=\frac{2 C}{(1+\Lambda) \mu}+\frac{(1+\Lambda) \mu \gamma}{2} \mathbb{E}\left[\left\|\ow-\w^{r*}\right\|^2\right]$, with $C$ specified in \cref{lemma:Rt}.
\end{lemma}
\begin{remark}
This lemma can be derived from {\citet[Theorem 1]{hu2022fair}}. 
It indicates the required unlearning rounds to reach a $\nu$-approximate saddle point of \textbf{P5}. 
 It highlights that increased data heterogeneity among remaining clients and stringent fairness constraints 
 require more unlearning rounds $T$ to balance the trade-off. 
\end{remark}

\begin{theorem}\label{thm:veri-fairness-guarantee}
Given $\epsilon = \frac{1}{\Lambda} (F_{-\mathcal{J}} (\ow) - F_{-\mathcal{J}}^*)$ 
and assuming the existence of $\nu$-approximate saddle points of the trade-off fairness problem, 
then the unlearning verification metric $V \leq 2 \nu =  4 \rho^2 \Lambda$ and $\max_{i \notin \mathcal{J}} \Delta f_i \leq \epsilon$. 
\end{theorem}
\cref{thm:veri-fairness-guarantee} emphasizes the feasibility of 
   $\nu$-approximate suboptimal solution that balances FU verification with fairness constraint $\epsilon$, {providing two insights:}
   \begin{enumerate}[nosep, leftmargin=12pt]
       \item \textit{Data Hetegeneity}: When the original FL model and optimal FU model are homogeneous, then $\epsilon = \frac{1}{\Lambda} \left(F_{-\mathcal{J}} (\ow) - F_{-\mathcal{J}} (\starw)\right)$ as $\starw$ is identical to $\w^{r*}$ under homogeneity.
 There is no fairness loss \textit{from unlearning} but from data heterogeneity among remaining clients (as discussed in \cref{thm:trade-off-lf-cu}). 
 However, under higher heterogeneity, characterized by large $F_{-\mathcal{J}} (\ow) - F_{-\mathcal{J}} (\starw)$, achieving $2\nu$ verification requires 
   FU \textit{compromises fairness} $\epsilon$.
   \item \textit{$\nu$ Selection}: Choosing a smaller $\nu$ potentially reduces $V$ but aggressively minimizing $\nu$ risks infeasibility 
   and increased resources for growing $T$ (stated in \cref{thm:fairness-trade-off}). 
   \end{enumerate}
    For future work, these insights suggest advanced strategies adjusting to data heterogeneity and system constraints.
\subsection{Discussion}
\revision{
This section demonstrates how our proposed mechanism enhances the adaptability and robustness of existing unlearning algorithms like rapid retraining \cite{wu2020deltagrad} and knowledge distillation \cite{wu2022federated}. 
Specifically, introducing $\lambda$ in \citet[Equation 3]{wu2020deltagrad} or \citet[Algorithm 1, Line 1]{wu2020deltagrad} allows managing stability and unlearning effectiveness. ($\lambda=1$ prioritizes unlearning, $\lambda=0$ reverts to original FL training). 
Similarly, by altering the objective function with the fairness-constrained objective in \textbf{P5}, our proposed mechanism incorporates fairness considerations into existing unlearning algorithms. 
}

\section{Experiemnts}\label{appx:exp}

\subsection{Experiment Settings}

\begin{wrapfigure}{r}{0.4\textwidth}
  \begin{center}
    \includegraphics[width=\linewidth]{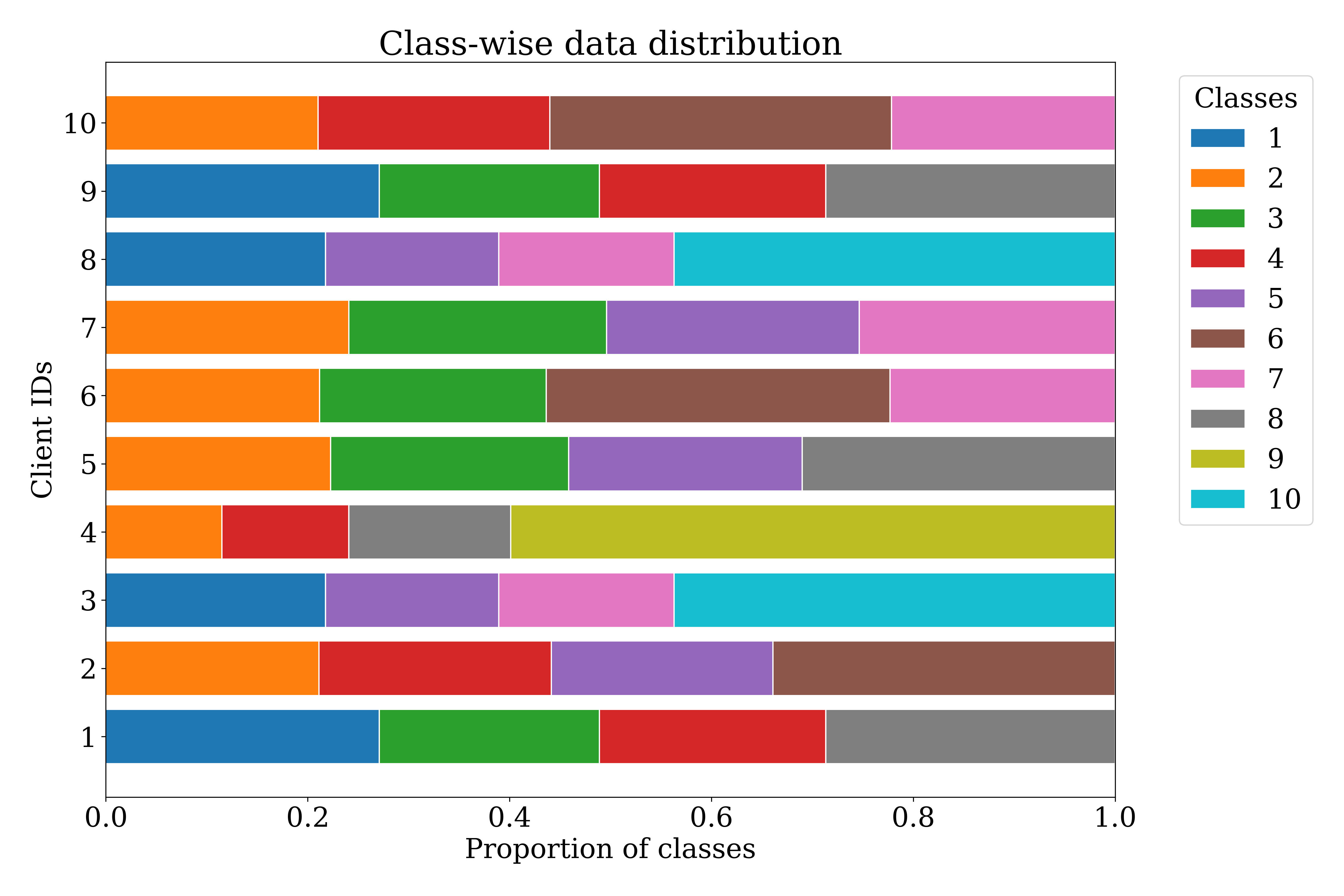}
\end{center}
    \caption{Non-IID Class Distribution for MNIST and CIFAR-10: Each client has four classes.}
    \label{fig:class distribution}  
\end{wrapfigure}

In our experiment, our implementation utilizes the open-source FL framework {Plato}~\cite{li2023plato} for reproducibility. We utilize the MNIST dataset~\cite{lecun2010mnist} and the CIFAR-10 dataset~\cite{Krizhevsky09learningmultiple} non-IID distributed across ten clients, each holding four distinct classes (the data distribution is detailed in~\cref{fig:class distribution}).
 
We employ LeNet-5 architecture~\cite{lecun1998gradient} and ResNet18 architecture~\cite{he2016deep}, classic non-convex neural network models, to evaluate FU's side effects.
To straightforwardly assess the FU evaluations metrics ($V, S, Q$), we focus on accuracy, \textit{e}.\textit{g}., 
$V = -Acc_{-\mathcal{J}} (\uw) + Acc_{-\mathcal{J}} (\boldw{}{r*})$ (in percentage). 
A smaller value of these metrics indicates better effectiveness, stability, or fairness achieved by our FU mechanism.
The overall experimental evaluation confirms our FU mechanisms in balancing the trade-offs, aligning with the theoretical insights in~\cref{sec:trade-offs}.

\subsection{FU for Balancing Stability}
We examine the stability penalty $\lambda$ in our FU mechanism in \cref{sec:sol-stab} and unlearning clients $[3, 4, 8]$ starts at round $10$ where the FL model has converged. 

\begin{figure}[t]
    \centering
    \begin{subfigure}[b]{0.45\textwidth}
        \centering
        \includegraphics[width=\linewidth]{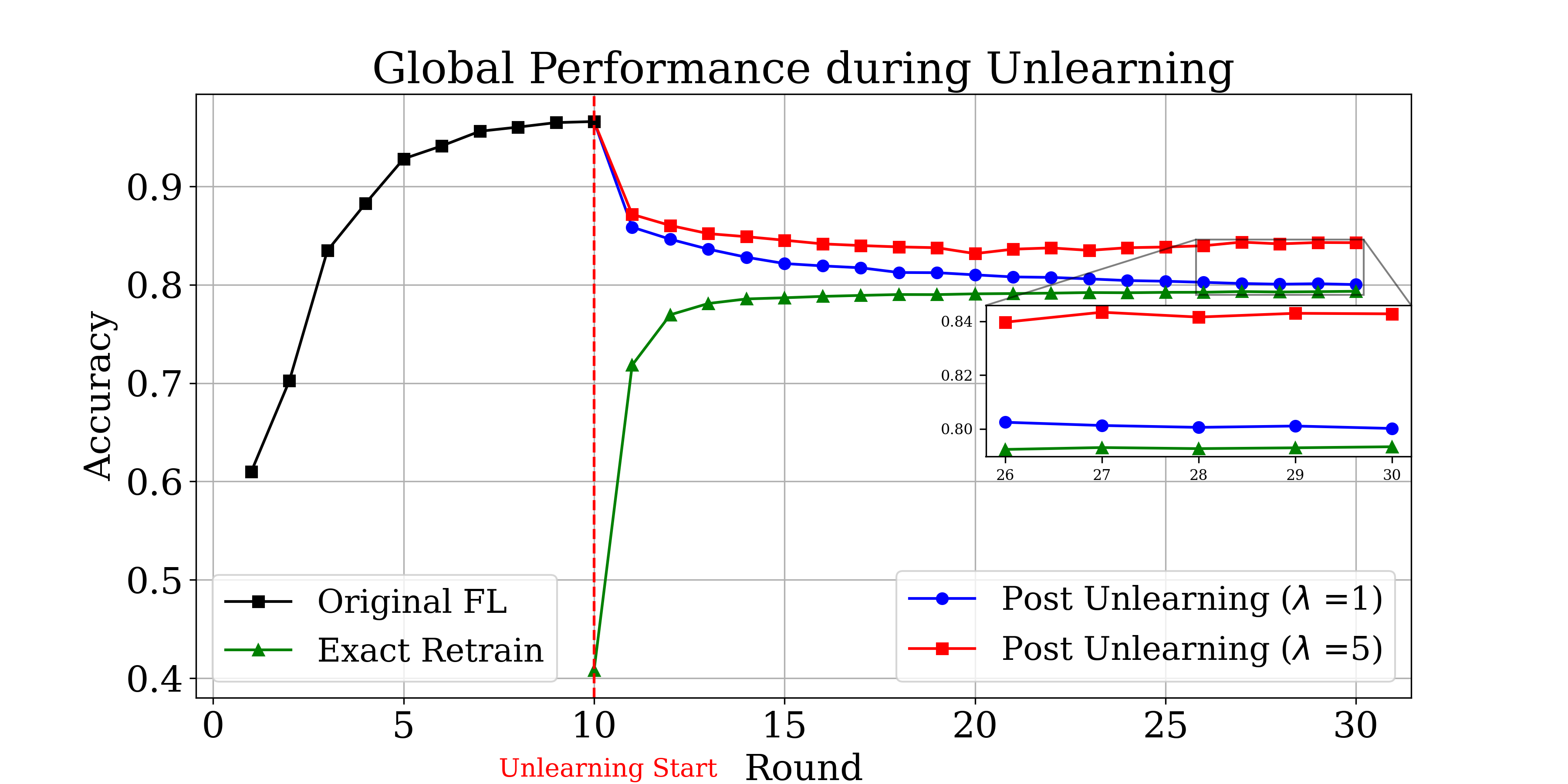}
        \caption{Unlearning Convergence and Handling Stability: Increasing stability penalty $\lambda$ enhances global performance of FU. (MNIST + LeNet-5)}
        \label{fig:stab_process}
    \end{subfigure}
    \hfill
    \begin{subfigure}[b]{0.45\textwidth}
        \centering
        \includegraphics[width=\linewidth]{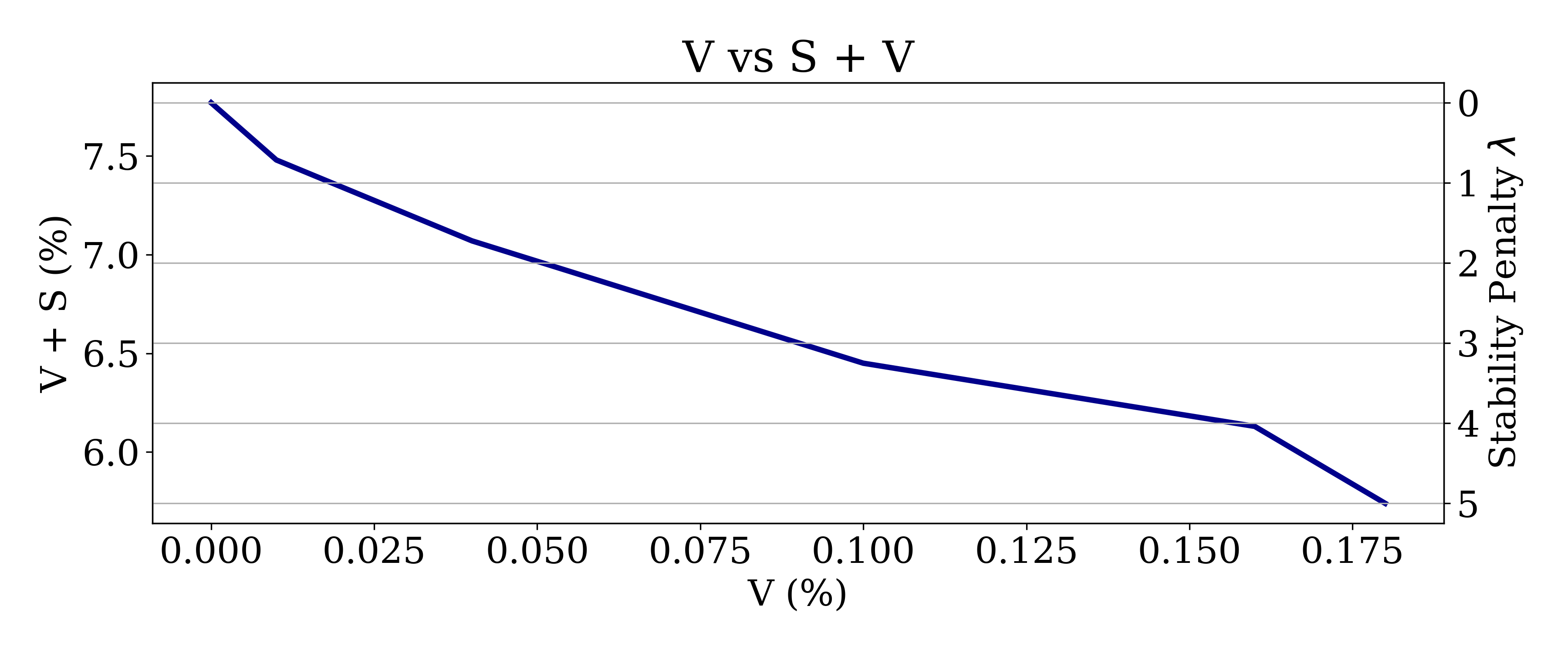}
        \caption{Trade-offs with Stability Penalty $\lambda$: Increasing $\lambda$ improves the balance between FU verification and stability, reflected by reduced $V + S$.}
        \label{fig:trade-offs_exp}
    \end{subfigure}
    \caption{Comparison of stability-related processes.}
    \label{fig:stability_comparison}
\end{figure}
\begin{table}[t]
    \centering
    \caption{Balancing Stability in FU (ResNet18, CIFAR10)}
    \label{tab:performance-comparison}
    \small
    \begin{tabularx}{\linewidth}{lllXXX}
        \hline
        & Global Performance (\%) & Unlearning Performance (\%) & V (\%) & S (\%) & V + S (\%) \\ \hline
        Original FL & 63.6 & - & - & - & - \\ 
        Retrain & 61.18 & 63.99 & 0 & 2.42 & 2.42 \\ 
        Unlearning ($\lambda = 1$) & 62.10 & 63.57 & 0.42 & 1.42 & \textbf{1.84} \\ 
       \hline
    \end{tabularx}
\vspace{-1em}
\end{table}

\textbf{Unlearning Convergence and Handling Stability}. 
The heterogeneity between remaining and unlearned clients leads to instability after unlearning (discussed in \cref{lem:stability-bound}), as indicated by the reduced global performance after retraining in \cref{fig:stab_process}.  
Additionally, \cref{fig:stab_process} showcases 
the convergence of our FU mechanism with different stability penalties $\lambda$ in the context of global performance.
Specifically, with a stability penalty $\lambda=1$, unlearning shows better stability than exact retraining (\cref{fig:stab_process}, \cref{tab:performance-comparison}). Increasing $\lambda$ to 5 further improves the stability of FU.

\textbf{Balancing Verification and Stability}: As shown in Figure~\ref{fig:trade-offs_exp}, increasing $\lambda$ lowers $V + S$, improving the balance between verification and stability in FU. 

Although a higher $\lambda$ reduces FU effectiveness (increasing $V$), our FU mechanism allows a better trade-off within a certain tolerance level for $V$. 
Additionally, it demonstrates time efficiency compared to retraining in~\cref{tab:unlearning_efficiency}, making it practical in real-world FU scenarios.

\begin{table}[t]
\caption{Unlearning Efficiency Compared to Retraining.({\textit{Computing Resources}: 2 Intel Xeon Gold 5217 CPUs, 384GB RAM, and 8 Nvidia GeForce RTX-2080Ti GPUs})}
\label{tab:unlearning_efficiency}
\begin{center}
\begin{small}
\begin{sc}
\begin{tabular}{cc}
\hline
& {Faster than Retrain ($\times$)} \\
\hline
\textbf{$\lambda = 1$} & 1.425$\pm$0.057 \\
\textbf{$\lambda = 3$} & 2.103$\pm$0.149 \\
\textbf{$\lambda = 5$} & 3.211 $\pm$0.315 \\
\hline
\end{tabular}
\end{sc}
\end{small}
\end{center}
\end{table}

\subsection{FU for Balancing Fairness}\label{sec:exp_fair}

\begin{wrapfigure}{r}{0.4\textwidth}
  \begin{center}
   \includegraphics[width=\linewidth]{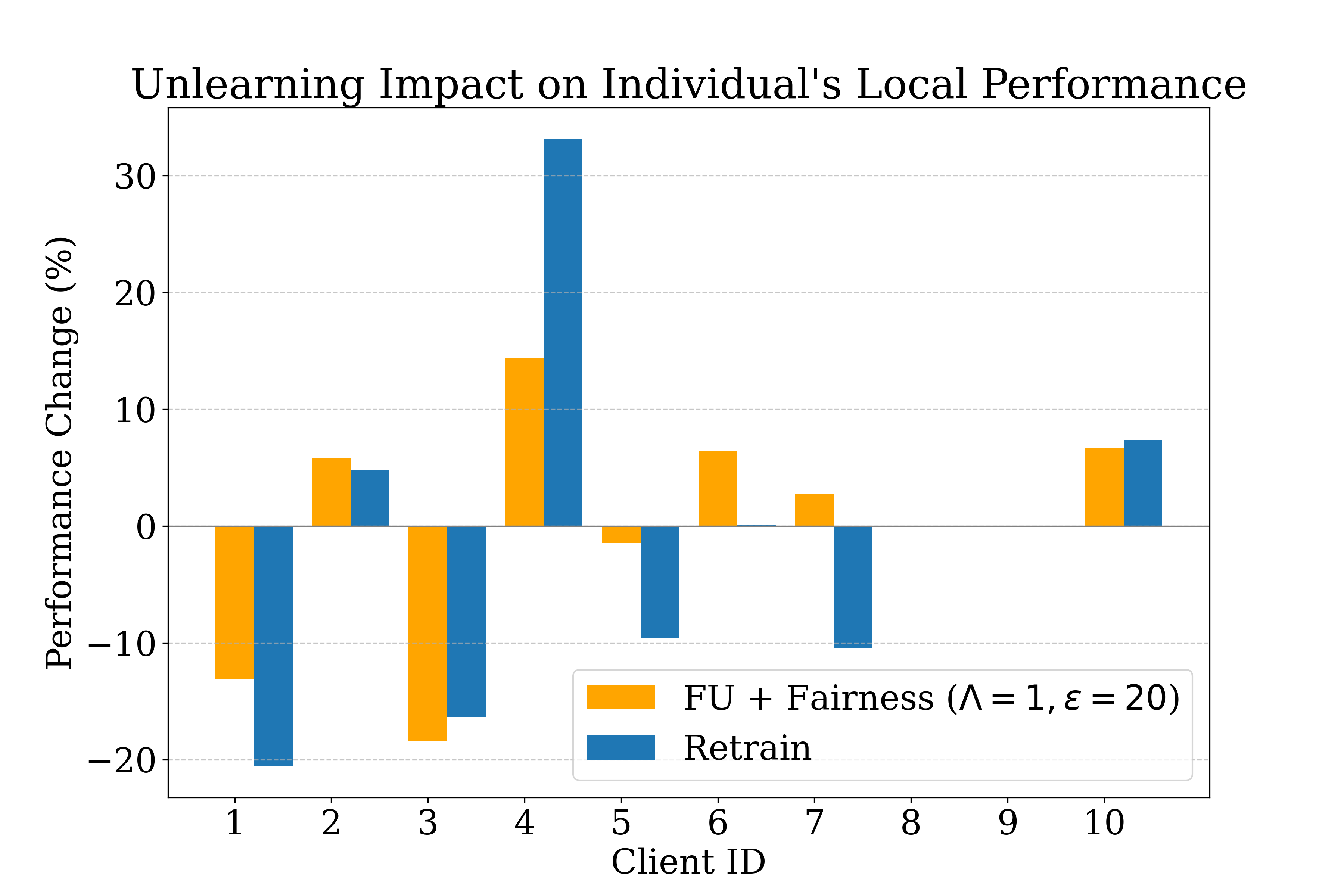}
\end{center}
    \caption{Unlearning Impact on Fairness (CIFAR10 + ResNet18): ensuring no client's performance deviates beyond $20\%$; fairness metric ($Q$) of our FU mechanism ($Q_{\text{our}} = 6.91$) is lower than that of retrain {$Q_{\text{retrain}} = 10.23$}}
    \label{fig:fair_cifar10}
\end{wrapfigure}

This section explores the fairness in FU under varying levels of heterogeneity \textit{among remaining clients}, and data between remaining and unlearned clients is also heterogeneous\footnote{According to \cref{thm:fairness-trade-off}, given homogeneous two groups, the fairness primarily depends on data heterogeneity among remaining clients, as we aim to investigate the impact of unlearned clients, we choose heterogeneous two groups setting.}. 
\cref{fig:fair_cifar10} illustrates the impact of unlearning on fairness among remaining clients, with a fairness parameter $\Lambda=1, \epsilon=20$. In this scenario, clients [1, 3] experienced significant utility loss due to their similar data distribution with unlearned clients [8, 9] (as shown in~\cref{fig:class distribution}).  
However, our FU mechanism in \cref{sec:sol-fair} achieves a lower fairness metric ($Q_{\text{our}} = 6.91$) compared to the retraining approach ($Q_{\text{retrain}} = 10.23$), indicating more equitable utility changes among remaining clients. The verification metric in this case is $V=0.42$, demonstrating that our mechanism enhances fairness even in more complex and heterogeneous environments. 
This enhancing FU effectiveness and fairness is consistent with insights on data heterogeneity in \cref{thm:fairness-trade-off} and \cref{thm:trade-off-lf-cu}.

\subsection{
Heterogeneity between remaining and unlearned clients}
This section investigates the stability in FU under varying levels of heterogeneity \textit{between these two groups}, and data within both remaining and unlearned clients is homogeneous. According to \cref{thm:fairness-trade-off}, this setting should facilitate fairness in the unlearning process. 
Specifically, the number of classes in clients' datasets follows a Dirichlet distribution with parameter $\alpha$, where a lower $\alpha$ value indicates higher data heterogeneity. We explore scenarios with $\alpha$ values corresponding to label distributions of 0.1, 0.4, and 0.7.
 
Table \ref{tab:exp_data_heterogeneity_between} demonstrates a correlation between the data heterogeneity level and the system's stability after unlearning.
A higher data heterogeneity leads to increased instability after unlearning, as indicated by the higher $S_{\text{retrain}}$ values. 
Notably, applying our FU mechanism with a stability trade-off parameter ($\lambda=1$) results in improved stability, as shown by the reduced $S_{\lambda=1}$ values, while the difference in FU verification remains negligible. This underscores our proposed approach could balance the trade-off between unlearning verification and global stability under varying degrees of data heterogeneity.

To further verify the unlearning, we examine the accuracy for class 4, which is unique to the unlearned clients. 
In the original trained model $\ow$, we observe a 77.09\% accuracy for class 4. However, in the retrained model $\w^{r}$, the accuracy for class 4 drops to 0\%, indicating successful unlearning. In our FU mechanism with stability penalty $\lambda=1$, the accuracy for class 4 is 0.3\%, further validating the effectiveness of our approach in unlearning the influence of unlearned clients while maintaining stability.

\begin{table}[t]
    \centering
    \caption{Data Heterogeneity Between Groups ($P_J = 0.38$): A lower label distribution indicates higher heterogeneity. Greater heterogeneity leads to increased instability after FU (evidenced by larger $S_{\text{retrain}}$ values). Employing a stability trade-off FU approach with $\lambda=1$ enhances stability (reduced $S_{\lambda=1}$) while maintaining negligible compromise in FU verification ($V_{\lambda=1}$).}
\label{tab:exp_data_heterogeneity_between}
    \begin{tabular}{cccc}
        Label Distribution & $S_{\text{retrain}}$ & $S_{\lambda=1}$ & $V_{\lambda=1}$\\
        \hline
        0.1 & 8.24 & 7.69 (\textbf{-0.55}) & 0.0\\
        0.4 & 3.45 & 2.05 (\textbf{-1.4}) & 0.0\\
        0.7 & 0.26 & 0.23 (\textbf{-0.03}) & 0.06\\
    \end{tabular}
\end{table}

\section{Conclusion}
In this study, we investigated the trade-offs in FU under data heterogeneity, focusing on balancing unlearning verification with global stability and local fairness. 
We proposed a novel FU mechanism grounded in a comprehensive theoretical framework with optimization strategies and penalty controls.
Our findings highlight the impacts of data heterogeneity in FU, paving the way for future research to explore adaptive FU mechanisms. 

\newpage

\bibliography{example_paper}
\bibliographystyle{ACM-Reference-Format}

\newpage
\appendix

\section{Proof of~\cref{lem:verification-bound}}
\begin{proof}
Given the FU verification metric
$V(\uw) = \mathbb{E} \left[ F_{-\mathcal{J}}(\uw)\right] - F_{-\mathcal{J}}(\boldsymbol{w}^{r*})$, 
we analyze the metric using iterative updates in the FU process. 

For each iteration $t+1$: 
\begin{align}\label{eq:iter}
    F_{-\mathcal{J}}(\boldsymbol{w}^{t+1}) - F_{-\mathcal{J}}(\boldsymbol{w}^{t})
    &\geq \left \langle \nabla F_{-\mathcal{J}}(\boldsymbol{w}^{t}), \boldsymbol{w}^{t+1} - \boldsymbol{w}^{t} \right \rangle + \frac{\mu}{2} \| \boldsymbol{w}^{t+1} - \boldsymbol{w}^{t} \|^2 \nonumber \\
    & =  
    \left \langle -\eta \boldsymbol{g}_{\mathcal{S}}^t, \nabla F_{-\mathcal{J}}(\boldsymbol{w}^{t}) \right \rangle 
    + \frac{\mu \eta^2}{2} \| \boldsymbol{g}_{\mathcal{S}}^t \|^2 \nonumber
\end{align}
where $ \boldsymbol{g}_{\mathcal{S}}^t=\sum_{i \in \mathcal{S}} \alpha_i \boldsymbol{g}_i^t$, and $\alpha_i$ is the aggregation weight of client $i$.

Taking expectations on both sides, we derive:
\begin{align}
\nonumber  \E{F_{-\mathcal{J}}(\boldsymbol{w}^{t+1})} - F_{-\mathcal{J}}(\boldsymbol{w}^{t})
    \geq \frac{\eta}{2} \underbrace{\left\| \boldsymbol{G}_\mathcal{S}^t 
    - \nabla F_{-\mathcal{J}}(\boldsymbol{w}^{t}) \right\|^2}_{A_1} \\
\underbrace{- \frac{\eta}{2} \| \nabla F_{-\mathcal{J}}(\boldsymbol{w}^{t}) \|^2 - \frac{\eta}{2} \left\| \boldsymbol{G}_\mathcal{S}^t \right\|^2 + \frac{\mu \eta^2}{2} \mathbb{E} \|  \boldsymbol{g}_{\mathcal{S}}^t \|^2}_{A_2}
\end{align}

where $\boldsymbol{G}_\mathcal{S}^t = \sum_{i \in \mathcal{S}} \alpha_i \boldsymbol{G}_i^t$, with $\boldsymbol{G}_i^t = \E{\boldsymbol{g}_i^t}$ representing the gradient for client $i$ at iteration $t$, while $\boldsymbol{g}_i^t$ denotes the stochastic gradient. 

Firstly, for $A_1$: 
\begin{align}\label{eq:a_1}
\left\|\boldsymbol{G}_\mathcal{S}^t - \nabla F_{-\mathcal{J}}(\boldsymbol{w^t})\right\|^2 
 \leq \mathbb{E}_i \left\|\boldsymbol{G}_i^t - \nabla F_{-\mathcal{J}}(\boldsymbol{w^t})\right\|^2  \underset{\cref{ass:data-heterogeneity}}{\leq} 
    \zeta_{\mathcal{S}}^2 + \beta_{\mathcal{S}}^2 \|\nabla F_{-\mathcal{J}}(\boldsymbol{w^t})\|^2
\end{align}
 
Let $\left\|\boldsymbol{G}_\mathcal{S}^t - \nabla F_{-\mathcal{J}}(\boldsymbol{w^t})\right\|^2 = \zeta_{t}^2 + \beta_{\mathcal{S}}^2 \|\nabla F_{-\mathcal{J}}(\boldsymbol{w^t})\|$, where $\zeta_{t}^2$ is specific for each round $t$ ($\zeta_{t}^2 \leq  \zeta_{\mathcal{S}}^2$).

Applying the triangle inequality, we derive the following relation for $A_1$: 
$ \| \boldsymbol{G}_\mathcal{S}^t \|^2 \geq \zeta_{t}^2 + (\beta_{\mathcal{S}}^2 - 1) \|\nabla F_{-\mathcal{J}}(\boldsymbol{w^t})\|^2$.

Then, for $A_2$, we expand it as follows:
\begin{equation}\label{eq:a_2}
A_2 = \frac{\mu \eta^2}{2}    \underbrace{\mathbb{E}\left\| 
\boldsymbol{G}_{\mathcal{S}}^t - \boldsymbol{g}_{\mathcal{S}}^t
\right\|^2 }_{=\sigma_t^2}
+ \frac{\eta}{2} \left( \eta \mu - 1 \right) 
\left \| \boldsymbol{G}_{\mathcal{S}}^t \right\|^2
- \frac{\eta}{2} \left \| \nabla F_{-\mathcal{J}}(\boldsymbol{w}^{t}) \right\|^2
\end{equation}

Now, by \cref{eq:a_1} and \cref{eq:a_2}, taking expectation on both sides of \cref{eq:iter}:
\begin{align}
\mathbb{E} \left[F_{-\mathcal{J}}(\boldsymbol{w}^{t+1})\right] - F_{-\mathcal{J}}(\boldsymbol{w}^{t})
    & \geq \frac{\eta}{2} \underbrace{\left( \zeta_t^2   + \beta_{\mathcal{S}}^2 \| \nabla F_{-\mathcal{J}}(\boldsymbol{w}^{t}) \|^2 \right)}_{= A_1} \\
    &\quad + \underbrace{\frac{\mu \eta^2}{2} \sigma_t^2
    + \frac{\eta}{2} \left(
    \eta \mu \beta_{\mathcal{S}}^2 - \eta \mu - \beta_{\mathcal{S}}^2
    \right) \| \nabla F_{-\mathcal{J}}(\boldsymbol{w}^{t}) \|^2  + \frac{\eta}{2} \left(\eta \mu - 1 \right) \zeta_t^2 }_{\leq A_2}   \nonumber \\
    &= \frac{\mu \eta^2}{2} \left(\beta_{\mathcal{S}}^2 - 1 \right) \| \nabla F_{-\mathcal{J}}(\boldsymbol{w}^{t}) \|^2 
    + \frac{\eta^2 \mu}{2}  \left(\sigma_t^2 + \zeta_t^2 \right)  \nonumber \\
    &\geq \underbrace{\frac{\mu \eta^2}{2} M \left(\beta_{\mathcal{S}}^2 - 1 \right)}_{B_1} \left(F_{-\mathcal{J}}(\boldsymbol{w}^{t}) - F_{-\mathcal{J}}(\boldsymbol{w}^{r*}) \right)
    + \frac{\eta^2 \mu}{2}  \left(\sigma_t^2 + \zeta_t^2 \right),
\end{align}
where $M = 2 \mu$ if $\beta_{\mathcal{S}} \geq 1$, else $M = 2 L$. 

Let $\bar{\sigma^2} =  \frac{1}{T} \sum_{t=1}^T \sigma_t^2$ and $\bar{\zeta^2} =\frac{1}{T} \sum_{t=1}^T  \zeta_t^2$. 
Thus, by iterative updates, we have:
\begin{equation}
    \mathbb{E} \left[F_{-\mathcal{J}}(\boldsymbol{w}^{T}) \right] - F_{-\mathcal{J}}(\boldsymbol{w}^{r*}) 
    \geq \left( B_1 + 1 \right)^T \left(F_{-\mathcal{J}}(\ow) - F_{-\mathcal{J}}(\boldsymbol{w}^{r*}) \right) 
    + \frac{\eta^2 \mu}{2} (\bar{\sigma^2} + \bar{\zeta^2}) \sum_{r = 1}^{T} \left( B_1 + 1 \right)^{T-r}
\end{equation}

Taking $\eta^2 = \frac{2}{\mu M T^2}$, we have \revision{$B_1 = \frac{\beta_{\mathcal{S}}^2 - 1}{T^2}$.} 

Now, considering two cases for $B_1$:

\revision{
\textbf{Case 1:} If $B_1 \leq -1$ (since $T^2 \geq 1$, Case 1 only holds when $T = 1$ and $\beta_{\mathcal{S}}^2 = 0$), we get:
\begin{equation}
    \mathbb{E} \left[F_{-\mathcal{J}}(\boldsymbol{w}^{T}) \right] - F_{-\mathcal{J}}(\boldsymbol{w}^{r*}) \geq  
    \frac{1}{M} \left( \bar{\sigma^2} + \bar{\zeta^2} \right),
\end{equation}
}

\revision{
\textbf{Case 2:} If $B_1 > -1$, then $(B_1 + 1)^T \geq 1 + TB_1$, 
which leads to the inequality:
\begin{equation}
    \mathbb{E} \left[F_{-\mathcal{J}}(\boldsymbol{w}^{T}) \right] - F_{-\mathcal{J}}(\boldsymbol{w}^{r*}) \geq 
    \left(1 + \frac{\beta_{\mathcal{S}}^2 - 1}{T}\right)
    \left(F_{-\mathcal{J}}(\ow) - F_{-\mathcal{J}}(\boldsymbol{w}^{r*}) \right)
    + \rho \left( \bar{\sigma^2} + \bar{\zeta^2} \right),
\end{equation}
where $\rho = \frac{\eta^2 \mu}{2 B_1}\left( \left(B_1 + 1 \right)^T -1 \right) \geq 0 $ for $\forall B_1 \geq -1$.
Therefore, we can derive $\rho \geq  \frac{\eta^2 \mu T}{2} = \frac{1}{MT}$
}

\revision{
Combining two cases and ${M} \leq {2L} $, we conclude
\begin{align}~\label{eq:combine_veri}
    \mathbb{E} \left[F_{-\mathcal{J}}(\boldsymbol{w}^{T}) \right] - F_{-\mathcal{J}}(\boldsymbol{w}^{r*}) 
    &\geq 
   \left(1 + \frac{\beta_{\mathcal{S}}^2 - 1}{T}\right)
    \left(F_{-\mathcal{J}}(\ow) - F_{-\mathcal{J}}(\boldsymbol{w}^{r*}) \right)
   + \frac{1}{2 L T} \left( \bar{\sigma^2} + \bar{\zeta^2} \right),
\end{align}
}

\end{proof}

\section{Proof of \cref{lem:stability-bound}}
\begin{proof}
    Given the stability metric $S(\uw)$, we express it as
\begin{equation}
    S(\uw) = \mathbb{E} \left[ F(\uw)\right] - F(\starw) = A_1 + A_2,
\end{equation}
where $A_1 = \mathbb{E} \left[ F(\uw)\right] - F(\ow)$ and $A_2 = F(\ow) - F(\starw) = \delta$. $A_2$ represents the empirical risk minimization (ERM) gap.

Firstly, to bound $A_1$, we utilize the convexity of $F$:
$A_1 \geq \left\langle \nabla F(\ow), \E{\uw - \ow} \right\rangle + \frac{\mu}{2} \E{\| \uw - \ow \|^2}$.

With $\uw - \ow = - \eta \sum_{ r = 1}^{T} \mathbf{g}_{\mathcal{S}}^r$, and $\mathbf{g}_{\mathcal{S}}^r$ being the aggregated stochastic gradient from the subset of remaining clients $\mathcal{S} \subseteq \mathcal{N} \setminus \mathcal{J}$, let $\bar{\mathbf{g}}_{\mathcal{S}} = \frac{1}{T}\sum_{ r = 1}^{T} \mathbf{g}_{\mathcal{S}}^r$, we have $\uw - \ow = -\eta T \bar{\mathbf{g}}_{\mathcal{S}}$.

Considering FL with remaining clients, the global objective is $F_{-\mathcal{J}} = \sum_{i \notin \mathcal{J}} p_i' f_i$, where $p_i' = \frac{p_i}{1-P_\mathcal{J}}$. 

For FL with unlearned clients, the global objective is $F_\mathcal{J} = \sum_{j \in \mathcal{J}} p_j' f_j$, where $p_j' = \frac{p_{j}}{P_\mathcal{J}}$. Then, $\nabla F(\boldsymbol{\cdot}) = (1-P_\mathcal{J}) \nabla F_{-\mathcal{J}} (\cdot) + P_\mathcal{J} \nabla F_\mathcal{J} (\cdot)$.

Expanding $A_1$, we get:
\begin{align}
    & = \frac{(1-P_\mathcal{J}) \eta T}{2} \E{\| \nabla F_{-\mathcal{J}}(\ow) - \bar{\mathbf{g}}_{\mathcal{S}} \|^2} + \frac{P_\mathcal{J} \eta T}{2} \E{\| \nabla F_\mathcal{J}(\ow) - \bar{\mathbf{g}}_{\mathcal{S}} \|^2} \nonumber \\
    & \quad - \frac{(1-P_\mathcal{J})\eta T}{2} \| \nabla F_{-\mathcal{J}}(\ow) \|^2 - \frac{P_\mathcal{J} \eta T}{2} \| \nabla F_\mathcal{J}(\ow) \|^2 \nonumber \\
    & \quad - \frac{\eta T}{2} \E{\|  \bar{\mathbf{g}}_{\mathcal{S}} \|^2} + \frac{\mu \eta^2 T^2}{2} \E{\| \bar{\mathbf{g}}_{\mathcal{S}} \|^2},
\end{align}
Under \cref{assm:client-proportion} where $1-P_\mathcal{J} \geq P_\mathcal{J}$, and \cref{ass:bounded-gradient} where gradient norm is bounded ($\| \nabla F \|^2 \leq G^2$), 
taking $T \geq \frac{\mu}{\eta^2}$, we can derive the lower bound for $A_1$:

\begin{align}
     A_1 & \geq 
    \frac{P_\mathcal{J} \eta T}{2} \| \nabla F_{-\mathcal{J}}(\ow) - \nabla F_\mathcal{J} (\ow) \|^2 \nonumber + \frac{\eta^2 T^2 - \mu T}{2} (\| \bar{\mathbf{g}}_{\mathcal{S}} \|^2 + \sigma^2 )  
    + \frac{\eta T}{2} G^2  \nonumber \\
    & \underset{T \geq \frac{\mu}{\eta^2}}{\geq}
    \frac{P_\mathcal{J} \eta T}{2} \| \nabla F_{-\mathcal{J}}(\ow) - \nabla F_\mathcal{J} (\ow) \|^2  
\end{align}

Therefore, we obtain the lower bound for $S$ in FU:
\begin{equation}~\label{eq:combine_stab}
     S(\uw) \geq 
    \frac{P_\mathcal{J} \eta T}{2} \| \nabla F_{-\mathcal{J}}(\ow) - \nabla F_\mathcal{J} (\ow) \|^2  
     + \delta
\end{equation}

\end{proof}

\section{Proof of \cref{thm:gradient-differences}}
\begin{proof}
By setting $\eta = \frac{1}{T} \sqrt{\frac{2}{\mu M}}$, we ensure the inequality in \cref{eq:combine_veri} is satisfied. Given the fact that $\frac{1}{M} \leq \frac{1}{2\mu}$, it follows that $\eta \leq \frac{1}{\mu T}$. Therefore, the inequality in \cref{eq:combine_stab} also holds, completing the proof.
\end{proof}

\section{Proof of \cref{thm:trade-off-lf-cu}}
\begin{proof}
Starting with the local fairness metric $Q(\uw)$, we have:
\begin{align}
    Q(\uw) &= \sum_{i \notin \mathcal{J}} p_i' \left| \Delta f_i - \sum_{i \notin \mathcal{J}} p_i' \Delta f_i \right| \nonumber \\
    &= \sum_{i \notin \mathcal{J}} p_i' \left| \Delta f_i - V + F_\mathcal{J}(\starw) - F_{-\mathcal{J}}^*\right| \nonumber \\
    &\geq \sum_{i \notin \mathcal{J}} p_i' \left| \Delta f_i\right| - V - \left(F_\mathcal{J}(\starw) - F_{-\mathcal{J}}^*\right) \nonumber \\
    &\geq -2V + F_{-\mathcal{J}}^* - \sum_{i \notin \mathcal{J}} p_i' f_i(\boldsymbol{w_i}^*).
\end{align}

The last inequality arises from:

\begin{align}
    \sum_{i \notin \mathcal{J}} p_i' \left| \Delta f_i\right| &= \sum_{i \notin \mathcal{J}} p_i' \left| \E{f_i(\uw)} - f_i(\starw)\right| \nonumber \\
    &= \sum_{i \notin \mathcal{J}} p_i' \left|\left(f_i(\starw) - f_i(\boldsymbol{w_i}^*) \right) -  \left( \mathbb{E} \left[f_i(\uw)\right] - f_i(\boldsymbol{w_i}^*)\right) \right| \nonumber \\
    &\geq \sum_{i \notin \mathcal{J}} p_i' \left(
    f_i(\starw) - f_i(\boldsymbol{w_i}^*)\right) - \sum_{i \notin \mathcal{J}} p_i'\left( 
    \mathbb{E} \left[f_i(\uw)\right]
     - f_i(\boldsymbol{w_i}^*)\right) \nonumber \\
    &\underset{(1)}{\geq} F_{-\mathcal{J}}(\starw) - \sum_{i \notin \mathcal{J}} p_i' f_i(\boldsymbol{w_i}^*) - \left( \mathbb{E} \left[ F_{-\mathcal{J}}(\uw)\right] - F_\mathcal{J}^* \right) \nonumber \\
    &= F_{-\mathcal{J}}(\starw) - \sum_{i \notin \mathcal{J}} p_i' f_i(\boldsymbol{w_i}^*) - V.
\end{align}

The inequality (1) is justified because $\sum_{i \notin \mathcal{J}} p_i' f_i(\boldsymbol{w_i}^*) \leq F_{-\mathcal{J}}^*$, as $\min_{\boldsymbol{w}} F_{-\mathcal{J}}(\boldsymbol{w}) \geq \sum_{i \notin \mathcal{J}} p_i' \min_{\boldsymbol{w}} f_i(\boldsymbol{w})$.

Therefore, 

\[
 Q(\uw) + 2V(\uw)  \leq F_{-\mathcal{J}}^* - \sum_{i \notin \mathcal{J}} p_i' f_i(\boldsymbol{w_i}^*).
\]
\end{proof}

\section{Balancing Stability Unlearning Algorithm Analysis}

\subsection{Proof for \cref{thm:converge_stab_bala}}\label{apx:proof_converge_stab_bala}
\begin{lemma}{(Lemma~1 of \cite{DBLP:conf/iclr/LiHYWZ20})}\label{lemma:Rt}
Let \cref{ass:data-heterogeneity} holds and $\eta_l \leq \frac{1}{4L}$, 
the remaining clients' contribution to unlearning for every round $t$ gives:
    \begin{equation}
        {\E{\left\|\barboldw{}{(t, E)} - \boldw{}{r*} \right\|^2}} \leq (1-\eta_l \mu) \E{\norm{ \barboldw{}{(t, 0)} - \boldw{}{r*}} } + \eta_l^2 C
    \end{equation}
    where $C=  \sigma^2 + 6 L \Gamma   + 8 (E-1)^2 \left(\zeta_{\mathcal{S}}^2 + (\beta_\mathcal{S}^2+1) G^2\right)$, $\Gamma = F_{-\mathcal{J}}^* - \sum_{i \in \mathcal{S}} \alpha_i F_i^*$ and $\w^{r*} = argmin_{\w} F_{-\mathcal{J}} (\w)$ is the optimal unlearned model. 
\end{lemma}

\begin{proof}
(Proof for \cref{thm:converge_stab_bala})

Suppose the FU process involves a total of $T$ unlearning rounds, and within each round $t$, each participating client $i$ engages in  $E$ local iterations.  
During local iterations, client $i$'s model at iteration $\tau$ ($\tau \leq E$) of round $t$ is denoted as $\boldw{i}{(t, \tau)}$. 
At the end of round $t$, the server aggregates to obtain the global model $\barboldw{}{(t, E)}$ and updates the global model by gradient correction as  $\barboldw{}{(t+1, 0)} = \barboldw{}{(t, E)} - \eta_g \boldsymbol{g}_c^t$. 

Then, we can express $\barboldw{}{(t+1, 0)} = \barboldw{}{(t, 0)} -  \eta_l \boldsymbol{g}_\mathcal{S}^t - \eta_g \boldsymbol{g}_c^t$, and we have: 

\begin{align}\label{eq:0}
    \E{\left\| \barboldw{}{(t+1, 0)} - \boldw{}{r*} \right\|^2}
     & = \E{\left\| \barboldw{}{(t, 0)} -  \eta_l \boldsymbol{g}_\mathcal{S}^t - \eta_g \boldsymbol{g}_c^t - \boldw{}{r*} \right\|^2} \nonumber \\
     & = \underbrace{\E{\left\|\barboldw{}{(t, 0)} -  \eta_l \boldsymbol{g}_\mathcal{S}^t - \boldw{}{r*} \right\|^2}}_{R^t} + \underbrace{\eta_g^2 \E{\left\| \boldsymbol{g}_c^t \right\|^2} - 2\eta_g \E{\left \langle \barboldw{}{(t, 0)} - \boldw{}{r*}, \boldsymbol{g}_c^t \right \rangle}}_{\Phi^t}
\end{align}

where $R^t={\E{\left\|\barboldw{}{(t, E)} - \boldw{}{r*} \right\|^2}} $ can be bounded by \cref{lemma:Rt}.

Then, 
we will bound the second term $\Phi^t$ in \cref{eq:0}. 
By Cauchy-Schwarz inequality and AM-GM inequality, we have
\( 
-2  \left \langle \boldsymbol{g}_c^t, \barboldw{}{(t, 0)} - \boldw{}{r*} \right \rangle \leq {\eta_g}\norm{\boldsymbol{g}_c^t} + \frac{1}{\eta_g}\norm{\barboldw{}{(t, 0)} - \boldw{}{r*}}    
\)

Thus, 
\begin{align}\label{eq:phit}
    \Phi^t & \leq \E{\norm{\barboldw{}{(t,  0)}- \boldw{}{r*}}} + 2 \eta_g^2 \E{\norm{\boldsymbol{g}_c^t}} \nonumber \\
   & \underset{\text{\cref{lemma:bounded_gradient_correction}}}{\leq} \E{\norm{\barboldw{}{(t,  0)}- \boldw{}{r*}}} + 
   2 \eta_g^2 \zeta^{\prime 2} \phi 
   + 2 \eta_g^2(\beta^{\prime 2} + 1)\phi   \norm{\nabla F_{-\mathcal{J}} (\barboldw{}{t, 0})} \nonumber \\
    & \underset{\text{\cref{ass:bounded-gradient}}}{\leq} \E{\norm{\barboldw{}{(t,  0)}- \boldw{}{r*}}} + 2 \eta_g^2 \zeta^{\prime 2} \phi
     + 
    2 \eta_g^2 \phi \Omega G^2
\end{align}
where
$\Omega = (\beta^{\prime 2} + 1)  $, and $\beta^{\prime 2}$ indicates the data heterogeneity between remaining and unlearned clients. 

Under \cref{lemma:Rt}, taking \cref{eq:phit} into \cref{eq:0} and letting $\Delta_{t} = \E{\norm{\barboldw{}{(t,  0)}- \boldw{}{r*}}}$,
we have: 

\begin{align}
    \Delta_{t+1} \leq \left(2 -\eta_l \mu\right) \Delta_t + \eta_l^2 B
\end{align}
where 
$B = \sigma^2 + 6 L \Gamma + 8 \left(\zeta_{\mathcal{S}}^2 + (\beta_\mathcal{S}^2+1) G^2\right) (E-1)^2 + 2 \phi (\frac{\eta_g}{\eta_l})^2 (\Omega \phi G^2 +  \zeta^{\prime 2})  $

Next, we will prove $\Delta_t \leq \frac{v}{\gamma+t}$ where $v=\max \left\{\frac{\beta^2 B}{\beta \mu-4},(\gamma+1) \Delta_1\right\}$.
For a diminishing stepsize, $\eta_l=\frac{\beta}{2(t+\gamma)}$ for some $\beta>\frac{1}{\mu}$ and $\gamma>0$ 
such that $\eta_l \leq \frac{1}{4 L}$. 
We prove $\Delta_t \leq \frac{v}{\gamma+t}$ by induction. 

Firstly, the definition of $v$ ensures that it holds for $t=1$. Assume the conclusion holds for some $t$, it follows that
$$
\begin{aligned}
\Delta_{t+1} & \leq\left(2-\eta_t \mu\right) \Delta_t+\eta_t^2 B \\
& \leq\left(2-\frac{\beta \mu}{2 (t+\gamma)}\right) \frac{v}{2 (t+\gamma) }+\frac{\beta^2 B}{4(t+\gamma)^2} \\
& =\frac{4 (t+\gamma)-4}{4(t+\gamma)^2} v+\left[\frac{\beta^2 B}{4(t+\gamma)^2}-\frac{\beta \mu-4}{4(t+\gamma)^2} v\right] \\
& \leq \frac{v}{t+\gamma+1}
\end{aligned}
$$

By the $L$-smoothness of $H(\cdot)$ and AM-GM inequality and Cauchy–Schwarz inequality, we can derive:
\begin{align*}
    \E{H\left(\barboldw{}{(t, 0)}\right)}& -H^* = \E{H\left(\barboldw{}{(t, 0)}\right)}- H(\boldw{}{r*}) +  H(\boldw{}{r*}) - H^* \\
 & \leq \left(\barboldw{}{(t, 0)}-\boldw{}{r*}\right)^{\top} \nabla H(\boldw{}{r*}) +  \frac{L}{2} \Delta_t + \frac{L}{2} \norm{\boldw{}{r*} - \w^{H*}} \\ 
 & \leq \frac{L}{2}  \Delta_t + \frac{1}{2L} \norm{ H(\boldw{}{r*})} + \frac{L}{2} \Delta_t + 
 \frac{L}{2} \norm{\boldw{}{r*} - \w^{H*}} \\
 & = \frac{1}{2L} \norm{P_{\mathcal{J}} \lambda\nabla F_{\mathcal{J}}(\ow) + L (\w^{r*} - \ow)}  + L \Delta_t + 
 \frac{L}{2} \norm{\boldw{}{r*} - \w^{H*}} \\
  & \leq \frac{1}{2L} \norm{P_{\mathcal{J}} \lambda\nabla F_{\mathcal{J}}(\ow)} + \frac{L}{2} \norm{\w^{r*} - \ow}  + L \Delta_t + 
 \frac{L}{2} \norm{\boldw{}{r*} - \w^{H*}} \\
  & \leq L \frac{v}{\gamma+t} + \frac{1}{2L} \norm{P_{\mathcal{J}} \lambda\nabla F_{\mathcal{J}}(\ow)} + \frac{L}{2} \left( \norm{\w^{r*} - \ow}  + \norm{\boldw{}{r*} - \w^{H*}}\right) \\
\end{align*}

\end{proof}

\subsection{Proof for \cref{thm:verifiable-unlearning}}
\revision{To prove \cref{lem:verification-bound}, we begin to introduce and prove the following additional lemmas}:

\textbf{Additional Lemmas}
\begin{lemma}\label{lemma:bounded_gradient_correction2}
    Under \cref{ass:removed_remaining_heterogeneity} and \cref{ass:gradient_connection}, the expected norm of the gradient correction term at round $t$ is bounded as follows:
    $$
\E{\norm{\boldsymbol{g_c} (\barboldw{}{(t, E)})}} \leq \phi \left( \zeta^{\prime 2}
+ (\beta^{\prime 2} + 1) \epsilon \norm{\nabla F_{-\mathcal{J}} (\barboldw{}{t, 0})}\right)
$$, where $\phi = \lambda^2 P_\mathcal{J}^2 (1+\cos^2 \theta)$. $\cos^2 \theta$ represents $\cos^2 \theta$ represents the similarity between the objectives of the remaining and unlearned clients.
\end{lemma}

\begin{lemma}[Per Round Unlearning]\label{lemma:post-unlearn-veri-perround}
    For each iteration $t$ in the unlearning process: 
    \begin{align*}
        & \E{F_{-\mathcal{J}} (\boldw{}{t+1}) - F_{-\mathcal{J}} (\boldw{}{t})} \\
        & \leq \underbrace{\E{\left\langle -\eta_l \boldsymbol{g}_{\mathcal{S}}^t, \nabla F_{-\mathcal{J}} (\boldw{}{t})  \right\rangle} + \frac{L \eta_l^2}{ 2} \E{\norm{ \boldsymbol{g}_{\mathcal{S}}^t}}}_{A_1: \text{remaining clients training}} \\
        & \quad + \underbrace{\E{\left\langle -\eta_g \boldsymbol{g_c}^t, \nabla F_{-\mathcal{J}} (\boldw{}{t})  \right\rangle} +  \frac{L \eta_g^2}{ 2} \E{\norm{\boldsymbol{g_c}^t}}}_{A_2: \text{global correction}} 
    \end{align*}  
\end{lemma}

\revision{\begin{lemma}[]\label{lemam:subset_bound_gradient}
    Assume \cref{ass:bounded-gradient} holds. Given a set of clients $\mathcal{S}$, it follows that
    \[
    \mathbb{E}\left\| 
\boldsymbol{G}_{\mathcal{S}}^t - \boldsymbol{g}_{\mathcal{S}}^t 
\right\|^2 \leq \sum_{i \in \mathcal{S}} \alpha_i^2 \sigma_{i, t}^2,
    \]
    where $\boldsymbol{G}_{\mathcal{S}}^t = \E{ \boldsymbol{g}_{\mathcal{S}}^t}$.
\end{lemma}
}

\begin{proof} (\cref{thm:verifiable-unlearning})

    Based on \cref{lemma:post-unlearn-veri-perround}, 
    we can decomposed the convergence of unlearning verficiation $V = \mathbb{E}[F_{-\mathcal{J}} (\boldsymbol{w}^{t+1}) - F_{-\mathcal{J}} (\boldsymbol{w}^{t})]$ into two primary components $A_1$ and $A_2$:
    \[  
\frac{1}{2} \E{F_{-\mathcal{J}} (\boldw{}{t+1}) - F_{-\mathcal{J}} (\boldw{}{t})} \leq A_1, \qquad
\frac{1}{2} \E{F_{-\mathcal{J}} (\boldw{}{t+1}) - F_{-\mathcal{J}} (\boldw{}{t}) }\leq A_2
    \]
    These components $A_1$ represent the impact of training with the remaining clients, and $A_2$ represents the global correction. 
    
    \textbf{Derivation for component $A_1$:}

    For $A_1$, it is related to unlearning with the remaining clients. 
    By iterative derivation, we have
    \begin{align*}
        &  \frac{1}{2} \E{F_{-\mathcal{J}} (\boldw{}{t+1})} - F_{-\mathcal{J}} (\boldw{}{t}) 
          \leq A_1 \\ & =
         \frac{\eta_l}{2} \E{\norm{\boldsymbol{g}_{\mathcal{S}}^t - \nabla F_{-\mathcal{J}} (\boldw{}{t})}}
         -  \frac{\eta_l}{2} {\norm{ \boldsymbol{G}_{\mathcal{S}}^t}} - \frac{\eta_l}{2} \norm{\nabla F_{-\mathcal{J}} (\boldw{}{t})}
        + \frac{L \eta_l^2}{2} \norm{\boldsymbol{g}_{\mathcal{S}}^t} \\
        & = \frac{\eta_l}{2} \E{\norm{\boldsymbol{g}_{\mathcal{S}}^t - \nabla F_{-\mathcal{J}} (\boldw{}{t})}}
       + \frac{L \eta_l^2}{2} \E{\norm{ \boldsymbol{G}_{\mathcal{S}}^t - \boldsymbol{g}_{\mathcal{S}}^t}}   \\
        & \qquad  + \frac{\eta_l}{2} (\eta_l L - 1) \norm{\boldsymbol{G}_{\mathcal{S}}^t} - \frac{\eta_l}{2} \norm{ \nabla F_{-\mathcal{J}} (\boldw{}{t})} \\
        & \leq_{\text{(\cref{ass:data-heterogeneity} \& \revision{\cref{lemam:subset_bound_gradient})}}} \frac{\eta_l}{2} \left(\zeta_\mathcal{S}^2 + \beta_\mathcal{S}^2\norm{\nabla F_{-\mathcal{J}} (\boldw{}{t})}\right) 
        + \frac{L \eta_l^2}{2} \sigma^2   \\
        & \qquad + \frac{\eta_l}{2} (\eta_l L - 1) (\beta_\mathcal{S}^2 + 1) \norm{ \nabla F_{-\mathcal{J}} (\boldw{}{t})} + \frac{\eta_l}{2} (\eta_l L - 1) \zeta_\mathcal{S}^2 \\
    \end{align*}
    \revision{, where $\sigma^2 = \sum_{i \in \mathcal{S}} \alpha_i^2 \sigma_{i, t}^2$ by \cref{lemam:subset_bound_gradient}.}

    Thus, 
    \[ 
         \E{F_{-\mathcal{J}} (\boldw{}{t+1})} - F_{-\mathcal{J}} (\boldw{}{t}) \leq  
         \left({\eta_l^2 L}(\beta_\mathcal{S}^2 + 1) - \frac{\eta_l}{2} \right) \norm{ \nabla F_{-\mathcal{J}} (\boldw{}{t})} 
         + {\eta_l^2 L} \left(\sigma^2 + \zeta_\mathcal{S}^2  \right)
    \]

    Taking $\eta_l = \frac{1}{L T}$ and $T \geq 2\beta_\mathcal{S}^2 + 2$ , we have:

    \begin{align}\label{eq:bound_a1_todo}
        \E{F_{-\mathcal{J}} (\boldw{}{T})} - F_{-\mathcal{J}} (\boldw{}{r*}) 
        & \leq \underbrace{\left(1 - \frac{1}{2LT} + \frac{\beta_\mathcal{S}^2 + 1 }{LT^2}\right)^T \left(F_{-\mathcal{J}}(\ow) - F_{-\mathcal{J}} (\boldw{}{r*})\right)}_{D_1} \nonumber \\
        & \qquad + \underbrace{2\frac{\sigma^2 + \zeta_\mathcal{S}^2}{ ( T - 2\beta_\mathcal{S}^2 -2)}
        \left( 1- \left(1 - \frac{T - 2\beta_\mathcal{S}^2 - 2 }{2LT^2}\right)^T \right)}_{D_2} \nonumber \\
    \end{align}

    \textbf{Bounding $D_2$}: We will employ
    the inequality $(1-a)^T \leq 1-aT$ for $a \leq 1$. 
    Here, $a = \frac{T - 2\beta_\mathcal{S}^2 - 2 }{2LT^2}$. 
    Then, we verify that $a \leq 1$ by considering $2 L T^2 - T + 2 \beta_\mathcal{S}^2 + 2 \geq 0$.
    
    If $L (\beta_\mathcal{S}^2 + 1) \geq \frac{1}{16}$, then, $a \leq 1$ holds for $T \geq 1$.
    If $L (\beta_\mathcal{S}^2 + 1) < \frac{1}{16}$, then, considering $T \geq \frac{1 + \Delta}{4L}$, where $\Delta = \sqrt{1 - 16 L (\beta_\mathcal{S}^2 + 1)}$, 
    we have $a \leq 1$.

    Considering $T \geq \frac{1 + \Delta}{4L}$ and  we obtain
    $\left(1 - \frac{T - 2\beta_\mathcal{S}^2 - 2 }{2LT^2}\right)^T \geq 1 - \frac{T - 2\beta_\mathcal{S}^2 - 2 }{2LT}$, and 
    
    \begin{equation}\label{eq:bound_a1_b2}
        \frac{1}{ ( T - 2\beta_\mathcal{S}^2 -2)}
    \left( 1- \left(1 - \frac{T - 2\beta_\mathcal{S}^2 - 2 }{2LT^2}\right)^T \right)
    \leq  \frac{1}{2LT}
    \end{equation} 

    By taking $T \geq \max\{2\beta_\mathcal{S}^2 + 2,\frac{1 + \Delta}{4L}\}$ with $\Delta = \sqrt{\max\{0,1 - 16 L (\beta_\mathcal{S}^2 + 1)\}}$,
   and integrating the bounds derived in \cref{eq:bound_a1_b2} into \cref{eq:bound_a1_todo}, 
   we have:
    \begin{align}
        \E{F_{-\mathcal{J}} (\boldw{}{T})} - F_{-\mathcal{J}} (\boldw{}{r*})
        & \leq
        \left(1 - \frac{1}{2LT} + \frac{\beta_\mathcal{S}^2 + 1 }{LT^2}\right)^T \left(F_{-\mathcal{J}}(\ow) - F_{-\mathcal{J}} (\boldw{}{r*})\right) 
        \nonumber \\ & \qquad + \frac{\sigma^2 + \zeta_\mathcal{S}^2 }{ LT} \nonumber \\
        & = 2 \chi_1.
    \end{align}
    Here, $\chi_1$ is defined as 
    $$\frac{1}{2}\left(1 - \frac{1}{2LT} + \frac{\beta_\mathcal{S}^2 + 1 }{LT^2}\right)^T \left(F_{-\mathcal{J}}(\ow) - F_{-\mathcal{J}} (\boldw{}{r*})\right) + 
    \frac{(\sigma^2 + \zeta_\mathcal{S}^2) }{2 LT}$$.

    \textbf{Derivation for component $A_2$:}
 
 By \cref{lemma:post-unlearn-veri-perround}: 
    
    \begin{align*}
        & \frac{1}{2}\E{F_{-\mathcal{J}} (\boldw{}{t+1}) - F_{-\mathcal{J}} (\boldw{}{t}) } \leq A_2 = \frac{L \eta_g^2}{2} \E{\norm{\boldsymbol{g_c} (\barboldw{}{(t, E)})}} -  \eta_g \E{\left\langle \boldsymbol{g_c} (\barboldw{}{(t, E)}), \nabla F_{-\mathcal{J}} (\barboldw{}{(t, 0)}) \right\rangle}  \nonumber \\
        & \quad = 
        {\frac{1}{2}({L \eta_g^2}-\eta_g)} \E{\norm{\boldsymbol{g_c} (\barboldw{}{(t, E)})}}  
        + \frac{\eta_g}{2} \E{\norm{\boldsymbol{g_c} (\barboldw{}{(t, E)})-\nabla F_{-\mathcal{J}} (\barboldw{}{(t, 0)})}} 
        - \frac{\eta_g}{2} \norm{\nabla F_{-\mathcal{J}} (\barboldw{}{(t, 0)})} \\
        & \quad \leq_{(1)} \frac{1}{2}{L \eta_g^2} \zeta^{''2} + \frac{1}{2}\left({L \eta_g^2}(\beta''^2 + 1) - 2\eta_g \right)\norm{\nabla F_{-\mathcal{J}} (\barboldw{}{(t, 0)})} \\
    \end{align*}
       where \( 
    \zeta^{''2} =  \phi \zeta^{\prime 2},   \beta^{'' 2} = \phi \epsilon \beta^{\prime 2} + \phi \epsilon + 1  
    \).
    
    The last inequality holds for: 
    \begin{align*}
        & \E{\norm{\boldsymbol{g_c} (\barboldw{}{(t, E)}) - \nabla F_{-\mathcal{J}} (\barboldw{}{(t, 0)})}} 
        \leq \E{\norm{\boldsymbol{g_c} (\barboldw{}{(t, E)})}} + \norm{\nabla F_{-\mathcal{J}} (\barboldw{}{(t, 0)})} \\
        & \underset{\text{\cref{lemma:bounded_gradient_correction2}}}{\leq}
        \phi \left( \zeta^{\prime 2}
        + (\beta^{\prime 2} + 1) \epsilon \norm{\nabla F_{-\mathcal{J}} (\barboldw{}{t, 0})}\right) + \norm{\nabla F_{-\mathcal{J}} (\barboldw{}{(t, 0)})}  \\
        & \leq \phi \zeta^{\prime 2}  + \left( \phi \epsilon \beta^{\prime 2} + \phi \epsilon + 1 \right) \norm{\nabla F_{-\mathcal{J}} (\barboldw{}{t, 0})} \\
        & \leq \zeta^{'' 2} + \beta^{'' 2} \norm{\nabla F_{-\mathcal{J}} (\barboldw{}{(t, 0)})}  
    \end{align*}
where \( 
    \zeta^{''2} =  \phi \zeta^{\prime 2},   \beta^{'' 2} = \phi \epsilon \beta^{\prime 2} + \phi \epsilon + 1  
    \).

    Choosing $\eta_g = \frac{1}{L T}$, similar to previous derivation for component $A_1$, we have: 
    
    \begin{align}\label{eq:bound_todo}
        \E{F_{-\mathcal{J}} (\boldw{}{T})} - F_{-\mathcal{J}} (\boldw{}{r*}) 
        & \leq 
        \underbrace{\left(1- \frac{2}{LT} + \frac{\beta^{''2} + 1 }{LT^2} \right)^T 
        \left(F_{-\mathcal{J}}(\ow) - F_{-\mathcal{J}} (\boldw{}{r*})\right)}_{B_1} \nonumber \\ 
        & \qquad + \underbrace{\frac{\zeta''^2\left(1-\left(1- \frac{(2T- \beta^{''2} - 1 )}{LT^2} \right)^T \right)}{2T - \beta''^2 - 1}}_{B_2} 
    \end{align}

    \textbf{Bounding $B_2$}: We will employ
    the inequality $(1-a)^T \leq 1-aT$ for $a \leq 1$. 
    Here, $a = \frac{2T - \beta^{''2} - 1 }{LT^2} $. 
    Then, we verify that $a \leq 1$ by considering $L T^2 - 2T +  \beta^{''2} + 1 \geq 0$.
    
    If $L (\beta''^2 + 1) \geq 1$, then, $a \leq 1$ holds for $T \geq 1$.
    If $L (\beta''^2 + 1) < 1$, then, considering $T \geq \frac{1 + \Delta'}{L}$, where $\Delta' = \sqrt{1 - L (\beta^{''2} + 1)}$, 
    we have $a \leq 1$.

    Considering $T \geq \frac{1 + \Delta'}{L}$ and  we obtain
    $\left(1- \frac{2T- \beta^{''2} - 1}{LT^2} \right)^T  \geq 1- \frac{2T- \beta^{''2} - 1}{LT} $, and 
    
    \begin{equation}\label{eq:bound_b2}
        {\frac{1-\left(1- \frac{2T- \beta^{''2} - 1}{LT^2} \right)^T}{2T - \beta''^2 - 1}}
    \leq  \frac{1}{LT}
    \end{equation} 

    By taking $T \geq \max\{\frac{1}{2}(\beta''^2 + 1),\frac{1 + \Delta'}{L}\}$, where $\Delta' = \sqrt{\max\{0, 1 - L (\beta^{''2} + 1)\}}$,
   
    and integrating the bounds derived in \cref{eq:bound_b2} into \cref{eq:bound_todo}, we have:
    \begin{align}
        \E{F_{-\mathcal{J}} (\boldw{}{T})} - F_{-\mathcal{J}} (\boldw{}{r*})
        & \leq
        \left(1- \frac{2}{LT} + \frac{\beta^{''2} + 1 }{LT^2} \right)^T\left(F_{-\mathcal{J}}(\ow) - F_{-\mathcal{J}} (\boldw{}{r*})\right) 
        \nonumber \\ & \qquad +  \frac{\zeta''^2}{LT} \nonumber \\
        & = 2 \chi_2.
    \end{align}
    Here, $\chi_2$ is defined as 
    $$\frac{1}{2}\left(1- \frac{2}{LT} + \frac{\beta^{''2} + 1 }{LT^2} \right)^T\left(F_{-\mathcal{J}}(\ow) - F_{-\mathcal{J}} (\boldw{}{r*})\right) + \frac{\zeta''^2}{2LT}$$.

\end{proof}

\subsection{Proof for \cref{lemma:bounded_gradient_correction}}\label{proof:lemma:bounded_gradient_correction}

\begin{proof}
    Recall that the gradient correction $\boldsymbol{g}_c^t = \boldsymbol{g}_c (\barboldw{}{(t, E)})$ at each round 
    $t$ after local $E$ epochs as:
    \(     \boldsymbol{g}_c^t =  \boldsymbol{h}^t - \text{Proj}_{\boldsymbol{g}_{-\mathcal{J}}} \boldsymbol{h}^t
    \), where $\boldsymbol{h}^t = \lambda (1-P_\mathcal{J}) \boldsymbol{g}_{-\mathcal{J}}^t   + \lambda P_\mathcal{J} \boldsymbol{\hat{g}}_\mathcal{J}^t$

    Thus, under stochastic gradient descent  $\boldsymbol{g}_{\mathcal{S}}^t$ from the subset of remaining client $\mathcal{S}$: 
    \[
        \boldsymbol{g_c}^t= \lambda P_\mathcal{J} \left(\boldsymbol{\hat{g}}_\mathcal{J}^t   -
          \cos \theta \frac{\| \boldsymbol{\hat{g}}_\mathcal{J}^t  \|}{ \| \boldsymbol{g}_{\mathcal{S}}^t   \|} \boldsymbol{g}_{\mathcal{S}}^t \right)
    \]
    
    The expected norm of $\boldsymbol{g_c}^t$ can be bounded as:
    \begin{equation}\label{eq:bound-gc}
    \E{\norm{\boldsymbol{g_c}^t}} \leq  \underbrace{\lambda^2 P_\mathcal{J}^2 (1+\cos^2 \theta)}_{\phi}\E{ \norm{ \boldsymbol{\hat{g}}_\mathcal{J}^t}},
    \end{equation}
    Therefore, by \cref{eq:bound-gc} and \cref{ass:removed_remaining_heterogeneity}, we obtain:
    \begin{align*}
        \E{\norm{\boldsymbol{g_c}^t}}    
           & \leq
            {\phi} \left( \zeta^{\prime 2} + 
            (\beta^{\prime 2} + 1) {\norm{\nabla F_{-\mathcal{J}} (\barboldw{}{t, E})}} \right) \\
    \end{align*}

    \end{proof}

\section{Balancing fairness Unlearning Algorithm Analysis}\label{apx:algo_fair}
 
\subsection{Proof for \cref{thm:veri-fairness-guarantee}}

Let $G = F_{-\mathcal{J}} (\w) + \boldsymbol{\lambda}^{\top} \boldsymbol{r}(\w)$,
and $(\bar{\w}$ and $\bar{\boldsymbol{\lambda}})$ is a $\nu$-approximate saddle point of $G$.

\begin{align*}
F_{-\mathcal{J}}(\bar{\w})+\Lambda\max _{z \in Z} \mathbf{r}_z(\bar{\w})_{+}-\nu & 
\leq F_{-\mathcal{J}}(\bar{w})+\overline{\boldsymbol{\lambda}}^T \mathbf{r}(\bar{w}) \\
& =G(\bar{w}, \overline{\boldsymbol{\lambda}}) \\
& \leq \min _w G(w, \overline{\boldsymbol{\lambda}})+\nu \\
& \leq G\left(\ow, \overline{\boldsymbol{\lambda}}\right)+\nu \\
& =F_{-\mathcal{J}} \left(\ow\right)+\nu \\
& \leq F_{-\mathcal{J}}(\ow)+\nu
\end{align*}

Hence,
$$
\begin{aligned}
\max _{z \in Z} \mathbf{r}_z(\bar{w})_{+} & \leq \frac{1}{\Lambda} \left(F_{-\mathcal{J}}(\ow)-F_{-\mathcal{J}}(\bar{\w})+2 \nu\right)
\\ & = \epsilon
\end{aligned}
$$

We can present $\nu$ as $\nu = \frac{\Lambda\epsilon + F_{-\mathcal{J}}(\bar{\w}) -  F_{-\mathcal{J}}(\ow) }{2}$.

Similarily, we can obtain $F_{-\mathcal{J}} (\bar{\w}) - F_{-\mathcal{J}}(\w^{r*}) \leq 2 \nu$.
By definition of unlearning verification, $V(\bar{\w}) \leq 2 \nu$.

That requires $\nu \leq \epsilon \Lambda+ F_{-\mathcal{J}}(\bar{\w}) -  F_{-\mathcal{J}}(\ow)$, 
which is equivalent to $\epsilon \geq \frac{F_{-\mathcal{J}}(\ow) - F_{-\mathcal{J}}(\w^{r*})}{\Lambda}$.

\section{Broader Impact and Limitations
}
This paper contributes to federated learning (FL) through federated unlearning (FU) techniques, addressing the growing need for user data rights and regulatory compliance. Our emphasis on fairness and stability in the unlearning process paves the way for more comprehensive and responsible FL systems development. Additionally, our detailed exploration of the impacts of data heterogeneity on the unlearning process offers insights into the development of more advanced FU methods. This focus ensures that our approach is not only theoretically sound but also practically applicable in diverse real-world scenarios.

\textit{Limitaions}: 
Our theoretical framework assumes a specific optimization formulation for the unlearning process, and the insights derived may not generalize to all FU settings or optimization techniques. Our focus has primarily been on the theoretical impacts. To enhance the generalizability of our findings, future research should investigate algorithm-centric methods through further experimentation across a wider range of scenarios. Additionally, developing more comprehensive frameworks that account for diverse factors in FU could provide a deeper understanding and more robust solutions.

\end{document}